\documentclass[journal]{IEEEtran}
\usepackage{makecell}
\usepackage{tabularx}
\usepackage{algorithm,listings}  
\usepackage{graphicx}
\usepackage{mathtools}
\usepackage[utf8]{inputenc} % allow utf-8 input
\usepackage[T1]{fontenc}    % use 8-bit T1 fonts
\usepackage{hyperref}       % hyperlinks
\usepackage{url}            % simple URL typesetting
\usepackage{booktabs}       % professional-quality tables
\usepackage{amsfonts}       % blackboard math symbols
\usepackage{nicefrac}       % compact symbols for 1/2, etc.
\usepackage{microtype}      % microtypography
\usepackage{xcolor}         % colors
\usepackage{algpseudocode}
\usepackage{amsmath,amssymb,amsfonts,textcomp}
\usepackage{mathtools}
\usepackage{natbib}
\usepackage{dsfont}
\usepackage{amsthm}
\usepackage{listings}
\usepackage{threeparttable}
\usepackage{pgfplots}
\pgfplotsset{width=15cm, height=10cm, compat=1.9}

\newtheorem{lemma}{Lemma}
\DeclareGraphicsExtensions{.pdf,.PDF,.jpeg,.JPEG,.jpg,.JPEG,.png,.PNG}

\usepackage{fancyhdr}
\fancypagestyle{ieee_copyright}{
  \fancyhf{}
  
  \fancyfoot[C]{\footnotesize \parbox{\textwidth}{\centering \copyright~2024 IEEE. Personal use of this material is permitted. Permission from IEEE must be obtained for all other uses, in any current or future media, including reprinting/republishing this material for advertising or promotional purposes, creating new collective works, for resale or redistribution to servers or lists, or reuse of any copyrighted component of this work in other works.}}
}
\ifCLASSINFOpdf
\else
\fi
\begin{document}
%
% paper title
% Titles are generally capitalized except for words such as a, an, and, as,
% at, but, by, for, in, nor, of, on, or, the, to and up, which are usually
% not capitalized unless they are the first or last word of the title.
% Linebreaks \\ can be used within to get better formatting as desired.
% Do not put math or special symbols in the title.
\title{Context-Enriched Contrastive Loss: Enhancing Presentation of Inherent Sample Connections in Contrastive Learning Framework}

% author names and affiliations
% transmag papers use the long conference author name format.

\author{\IEEEauthorblockN{Haojin Deng,~\IEEEmembership{Student Member,~IEEE,}
Yimin Yang,~\IEEEmembership{Senior Member,~IEEE}}
%\IEEEauthorblockA{\IEEEauthorrefmark{1,2}Department of Electrical and Computer Engineering, Western University, London, Canada.}
% <-this % stops an unwanted space
\thanks{This work was supported in part by the Natural Sciences and Engineering Research Council of Canada (NSERC) Discovery Grant Program under Grant RGPIN-2020-04757, in part by the NSERC Alliance Program under Grant ALLRP 583061-23, and in part by the Western Research Award.}
\thanks{H. Deng is with the Department of Electrical and Computer Engineering, Western University, London, Canada.}  
\thanks{Y. Yang is with the Department of Electrical and Computer Engineering, Western University, London, Canada, and also with the Vector Institute, Toronto, Canada}
\thanks{The code for this paper is available at: \url{https://github.com/hdeng26/Contex}.}
}

% The paper headers
\markboth{IEEE Transactions on Multimedia,~Vol.~27, pp.~429--441, December~2024}%
{Deng \MakeLowercase{\textit{and}} Yang: Context-Enriched Contrastive Loss}
% The only time the second header will appear is for the odd numbered pages
% after the title page when using the twoside option.
% 
% *** Note that you probably will NOT want to include the author's ***
% *** name in the headers of peer review papers.                   ***
% You can use \ifCLASSOPTIONpeerreview for conditional compilation here if
% you desire.

% If you want to put a publisher's ID mark on the page you can do it like
% this:
%\IEEEpubid{0000--0000/00\$00.00~\copyright~2015 IEEE}
% Remember, if you use this you must call \IEEEpubidadjcol in the second
% column for its text to clear the IEEEpubid mark.

% use for special paper notices
%\IEEEspecialpapernotice{(Invited Paper)}

% for Transactions on Magnetics papers, we must declare the abstract and
% index terms PRIOR to the title within the \IEEEtitleabstractindextext
% IEEEtran command as these need to go into the title area created by
% \maketitle.
% As a general rule, do not put math, special symbols or citations
% in the abstract or keywords.
\IEEEtitleabstractindextext{%
\begin{abstract}
Contrastive learning has gained popularity and pushes state-of-the-art performance across numerous large-scale benchmarks. In contrastive learning, the contrastive loss function plays a pivotal role in discerning similarities between samples through techniques such as rotation or cropping. However, this learning mechanism can also introduce information distortion from the augmented samples. This is because the trained model may develop a significant overreliance on information from samples with identical labels, while concurrently neglecting positive pairs that originate from the same initial image, especially in expansive datasets. This paper proposes a context-enriched contrastive loss function that concurrently improves learning effectiveness and addresses the information distortion by encompassing two convergence targets. The first component, which is notably sensitive to label contrast, differentiates between features of identical and distinct classes which boosts the contrastive training efficiency. Meanwhile, the second component draws closer the augmented samples from the same source image and distances all other samples, similar to self-supervised learning. We evaluate the proposed approach on image classification tasks, which are among the most widely accepted 8 recognition large-scale benchmark datasets: CIFAR10, CIFAR100, Caltech-101, Caltech-256, ImageNet, BiasedMNIST, UTKFace, and CelebA datasets.  The experimental results demonstrate that the proposed method achieves improvements over 16 state-of-the-art contrastive learning methods in terms of both generalization performance and learning convergence speed.  Interestingly, our technique stands out in addressing systematic distortion tasks. It demonstrates a 22.9\% improvement compared to original contrastive loss functions in the downstream BiasedMNIST dataset, highlighting its promise for more efficient and equitable downstream training. 
\end{abstract}

% Note that keywords are not normally used for peer-review papers.
}

% make the title area
\maketitle
\thispagestyle{ieee_copyright}

% To allow for easy dual compilation without having to reenter the
% abstract/keywords data, the \IEEEtitleabstractindextext text will
% not be used in maketitle, but will appear (i.e., to be "transported")
% here as \IEEEdisplaynontitleabstractindextext when the compsoc 
% or transmag modes are not selected <OR> if conference mode is selected 
% - because all conference papers position the abstract like regular
% papers do.
\IEEEdisplaynontitleabstractindextext
% \IEEEdisplaynontitleabstractindextext has no effect when using
% compsoc or transmag under a non-conference mode.

% For peer review papers, you can put extra information on the cover
% page as needed:
% \ifCLASSOPTIONpeerreview
% \begin{center} \bfseries EDICS Category: 3-BBND \end{center}
% \fi
%
% For peerreview papers, this IEEEtran command inserts a page break and
% creates the second title. It will be ignored for other modes.
\IEEEpeerreviewmaketitle

\section{Introduction}
% The very first letter is a 2 line initial drop letter followed
% by the rest of the first word in caps.
% 
% form to use if the first word consists of a single letter:
% \IEEEPARstart{A}{demo} file is ....
% 
% form to use if you need the single drop letter followed by
% normal text (unknown if ever used by the IEEE):
% \IEEEPARstart{A}{}demo file is ....
% 
% Some journals put the first two words in caps:
% \IEEEPARstart{T}{his demo} file is ....
% 
% Here we have the typical use of a "T" for an initial drop letter
% and "HIS" in caps to complete the first word.
\IEEEPARstart{R}{ecent} advances in traditional supervised learning have been largely dependent on high-quality data samples, extensive human labeling, and large models~\cite{CNN,DeepLearning,resnet,Transformers}. This dependency has escalated resource demands, prompting a shift towards efficient model performance, mindful of carbon footprints~\cite{metaPOT} and reduced batch sizes~\cite{Simsiam}. Innovations aimed at reducing resource consumption include introducing non-iterative strategies for retraining neurons in fully connected layers~\cite{pinv} and transferring knowledge across different modalities~\cite{eeg_dcnet}, which both contribute significantly to training efficiency. Self-supervised learning (SSL) has emerged as a promising solution, generating its own 'labels' from training samples, offering robustness and efficient data utilization, especially for imperfect datasets. In this realm, contrastive self-supervised learning distinguishes itself by leveraging similarities and dissimilarities between data pairs. For instance, MoCo~\cite{MOCO} champions unsupervised representation in computer vision, outclassing many in supervised pretraining, while innovatively using a queue for negative sample features. In contrast, SimCLR~\cite{SimCLR} adopts a more straightforward approach, emphasizing larger batch sizes and introducing a nonlinear projection head post-encoder, which has been demonstrated to significantly enhance representation quality. Contrastive Language-Image Pre-Training (CLIP)~\cite{clip} and Dual-Stream Contrastive Learning~\cite{zeroshotdual} for Compositional Zero-Shot Recognition have further refined contrastive learning, addressing the challenges of robustness and generalization in multimodal and compositional contexts.

\begin{figure}[t]
  \centering
  \includegraphics[width=1\columnwidth]{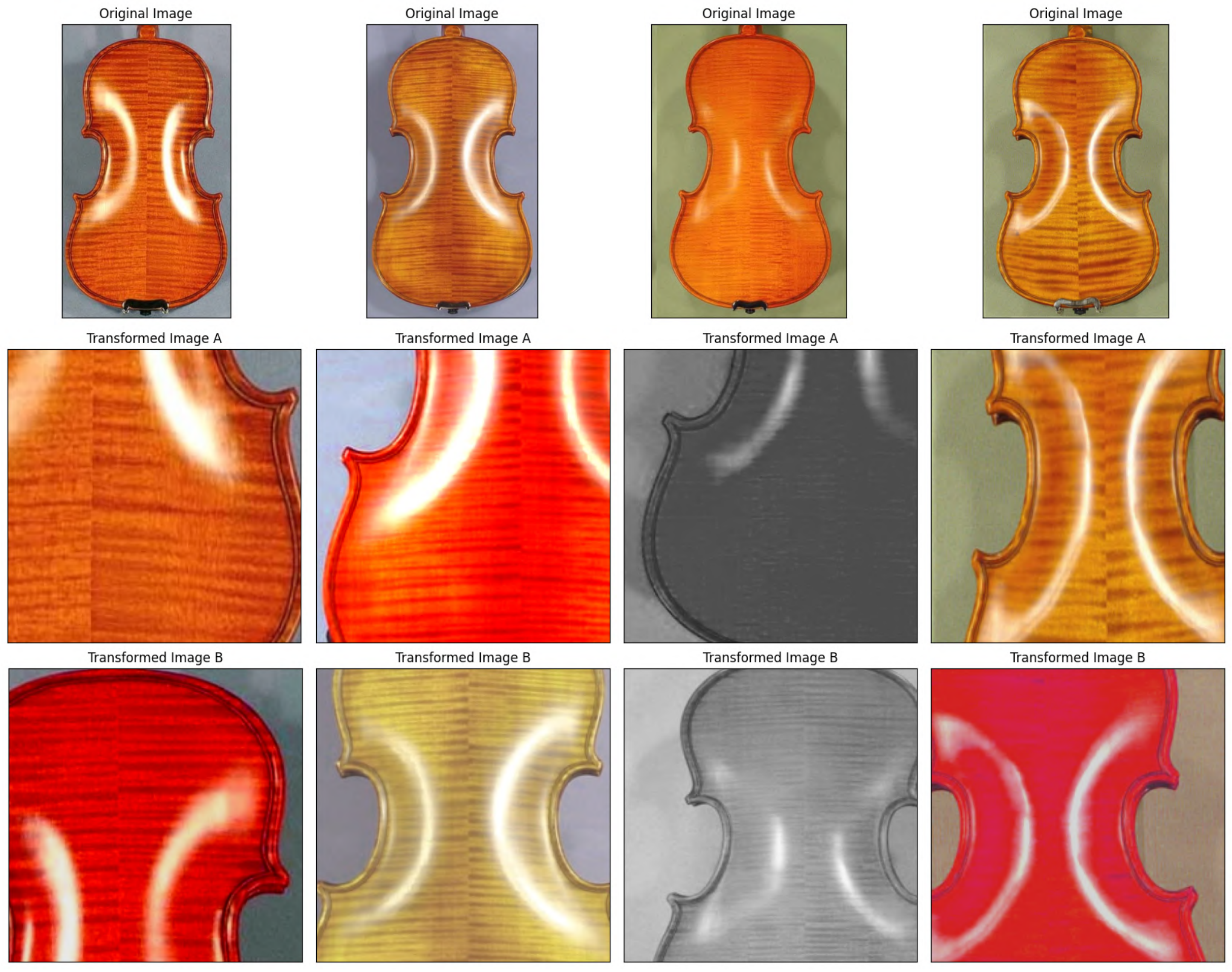}
  \caption{Random augmented images from ImageNet dataset. All these images are from the `n04536866' category. The images in the first row are the original images. The images in the same column are from the same source image.}
  \label{fig:mislead_aug}
\end{figure}

Building upon the principles of contrastive learning, metric distance representation learning has emerged as a pivotal approach in defining distances between data samples using labels~\cite{DM_large_margin, DM_Relative_Comparisons, rank_sim,NT-Xent,triplet_loss}. In particular, the method in~\cite{DM_large_margin} ensures that k-nearest neighbours belong to the same class, while the triplet loss~\cite{triplet_loss} differentiates between positive and negative samples. Innovations such as the loss of $(N +1)$-tuplet and the loss of $N$-pairs by~\cite{NT-Xent} have proven superior to traditional loss functions. 

Transitioning from these metric learning advancements, SupCon's supervised contrastive learning~\cite{supcon} integrates label information, enhancing robustness and outperforming frameworks like SimCLR~\cite{SimCLR}. Its potential is also seen in applications like personalized recommendation systems, where the Personalized Contrastive Loss (PCL) and Multisample-Based Contrastive Loss (MSCL) have demonstrated significant improvements in personalized representation and top-k recommendation tasks, respectively~\cite{pcl, mscl}. Despite its strengths, current supervised contrastive learning techniques still face challenges.

\begin{figure*}[!t]
  \centering
  \includegraphics[width=2\columnwidth]{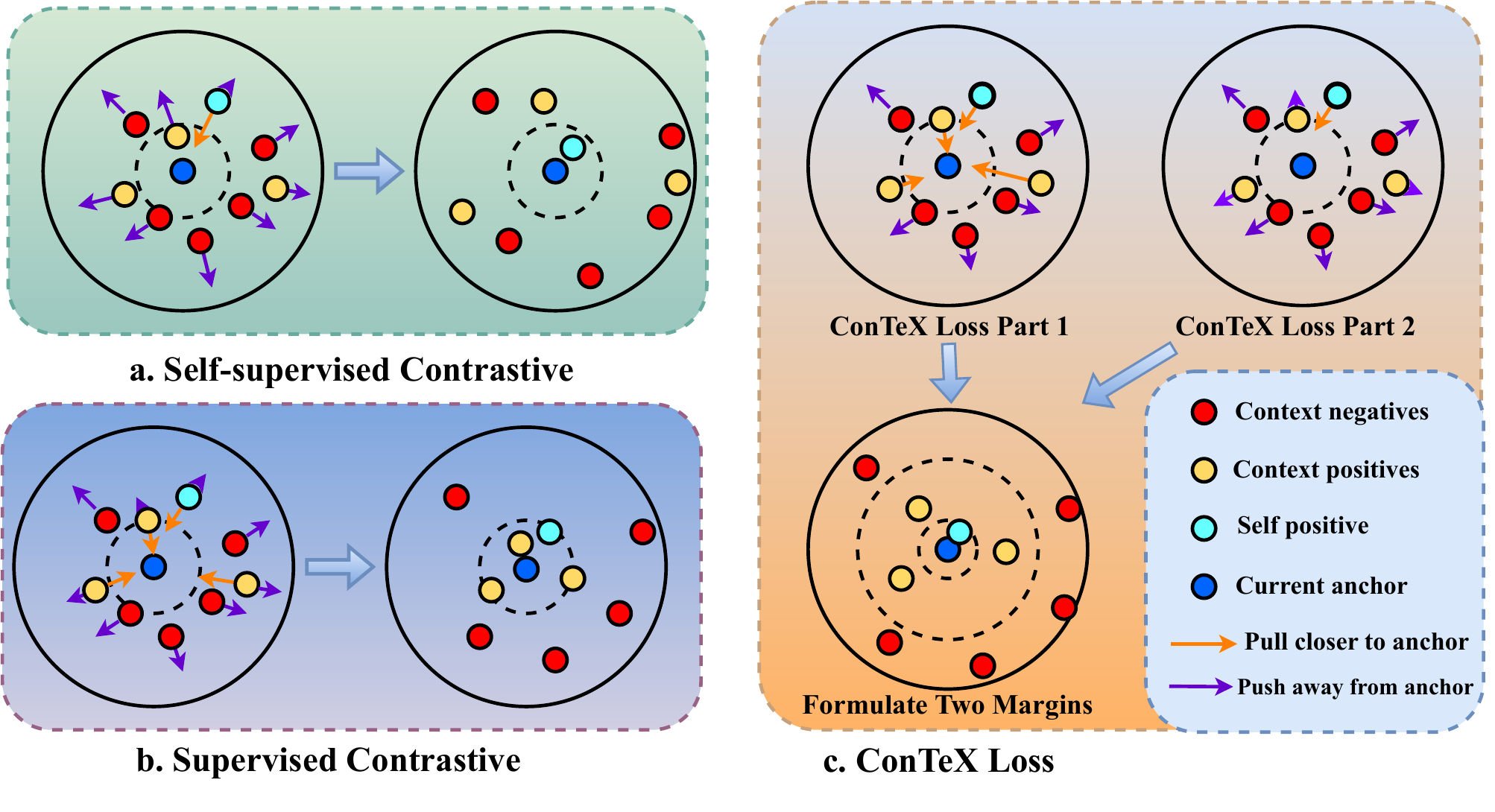}
  \caption{\textcolor{black}{Our proposed loss function vs previous contrastive loss functions in latent space. The \textbf{blue dot} is the current \textit{anchor}. The \textbf{red dots} are \textit{context negatives} (samples have different labels or context with \textit{anchor}). The \textbf{yellow dots} are \textit{context positives} (samples have the same label or similar context with \textit{anchor}). The \textbf{light cyan dot} is the \textit{self positive} (the augmented sample from the same original sample as the \textit{anchor}). The first part of our loss function is similar to (\textbf{b}), but we enhanced the contrastive between the \textbf{context positives} and \textbf{context negatives} (Eq.~\ref{loss part 1}). The second part of our loss function (Eq.~\ref{loss part 2}) has a similar proposal as (\textbf{a}). Our combined loss function created two stable boundaries to maximum similarity only from \textit{self positive} sample to avoid misleading by labels.}}
  \label{fig:loss}
\end{figure*}

The primary challenge in contrastive learning is the substantial number of epochs required for convergence during the pretraining stage, indicating untapped potential for enhancing the efficiency of the contrastive learning loss function. Previous supervised contrastive loss function~\cite{supcon} leverage context positives (samples belonging to the same label class) against self negatives (samples that are not augmented from the same image, denoted in Fig.~\ref{fig:loss}). Given that context-positive samples are inherently included within self-negatives, this setup inadvertently introduces an inherent tension in the loss function, thereby decelerating the convergence rate. The repercussions of this slow convergence are multifaceted: it not only increases computational costs and extends training times, but it may also compromise the adaptability of models in real-world scenarios where rapid model deployment is crucial.

On the other hand, systematic distortion due to incorrect assumptions that occurred during learning similarities still exists in contrastive learning frameworks. Motivated by supervised contrastive learning, various studies have integrated contrastive techniques into supervised learning, focusing on debiasing methods~\cite{chuang_debiased_2020, bahng_learning_2020, hong_unbiased_2021, barbano_unbiased_2022}. For instance, Chuang et al.~\cite{chuang_debiased_2020} introduced a debiased contrastive loss, while Bahng et al.~\cite{bahng_learning_2020} presented ReBias, a framework for debiased representation. Meanwhile, Li and Vasconcelos~\cite{hong_unbiased_2021} and \cite{barbano_unbiased_2022} introduced innovative loss functions and theoretical frameworks for debiasing. Yet, a shared challenge among these frameworks is their dependence on augmentations. Augmentation, combined with context-based positive samples, might inhibit learning inherent features, impacting generalization, especially in transfer tasks, as illustrated in Fig.~\ref{fig:mislead_aug}. Context-positive pairs from distinct original images might be more analogous than self-positive pairs, introducing potential bias as labeled samples guide network training. While many studies have explored supervised contrastive learning for bias mitigation~\cite{bahng_learning_2020,chuang_debiased_2020,hong_unbiased_2021,barbano_unbiased_2022}, the impact of `shortcuts' during augmentation, which can influence performance and training efficiency, remains under-explored.

In order to address the above two challenges, we introduce the Context-enriched Contrastive Loss Function (ConTeX) to enhance supervised contrastive learning. The abbreviation `ConTeX' is derived from the Latin word `contextus,' which means `woven together' or `context.' This name reflects our approach's emphasis on incorporating contextual information into the contrastive learning framework. Our loss function is bifurcated into two parts that not only mitigate the biases arising from context-positive samples and augmentations but also enhance the convergence speed during the pretraining phase.
The first part emphasizes discerning between similar-context samples (context positives in Fig.~\ref{fig:loss}) and different-context samples (context negatives in Fig.~\ref{fig:loss}). Unlike the SupCon loss, we've refined this segment to directly contrast positive and negative similarities, leading to faster convergence.
The second part focuses on the anchor and its `self positive', positioning all other samples as `self negatives'. This ensures the closest proximity between the anchor and its `self positive', bolstering generalization and curbing biases during training. 
Fig.~\ref{fig:loss} also delineates the distinctions between the ConTeX and traditional contrastive loss functions. Our modifications not only expedite the learning process but also enhance the model's robustness against potential biases. Overall, our contributions are listed as follows:
\begin{itemize}

\item Superior performance and faster convergence. We proposed a novel combination of two loss functions that utilized context-enriched contrastive information and avoid conflicts. Our proposed ConTeX loss provided an additional constraint to further restrict the distance of self-positive, context-positive and context-negative pairs which heightened sensitivity towards label contrasting, driving an accelerated convergence without escalation of computational complexity during the pretraining phase. This feature underscores the practicality of ConTeX in scenarios with limited computation resources.
\item Promoting fairness. We address an overlooked aspect in the field of supervised contrastive learning: the fairness information typically captured in unsupervised contrastive learning but somewhat diminished in supervised contrastive learning. Our focus extends beyond merely leveraging context information to include a comprehensive understanding of self-positive pair similarities. Consequently, our empirical results outshine those of SupCon, showcasing superior generalization capabilities, thereby reinforcing the effectiveness of our proposed approach.
\end{itemize}

\section{Methodology}

% base loss comparison image
% use the first part to compare with SupCon and SimCLR
\begin{figure*}[h]
  \centering
  \includegraphics[width=2\columnwidth]{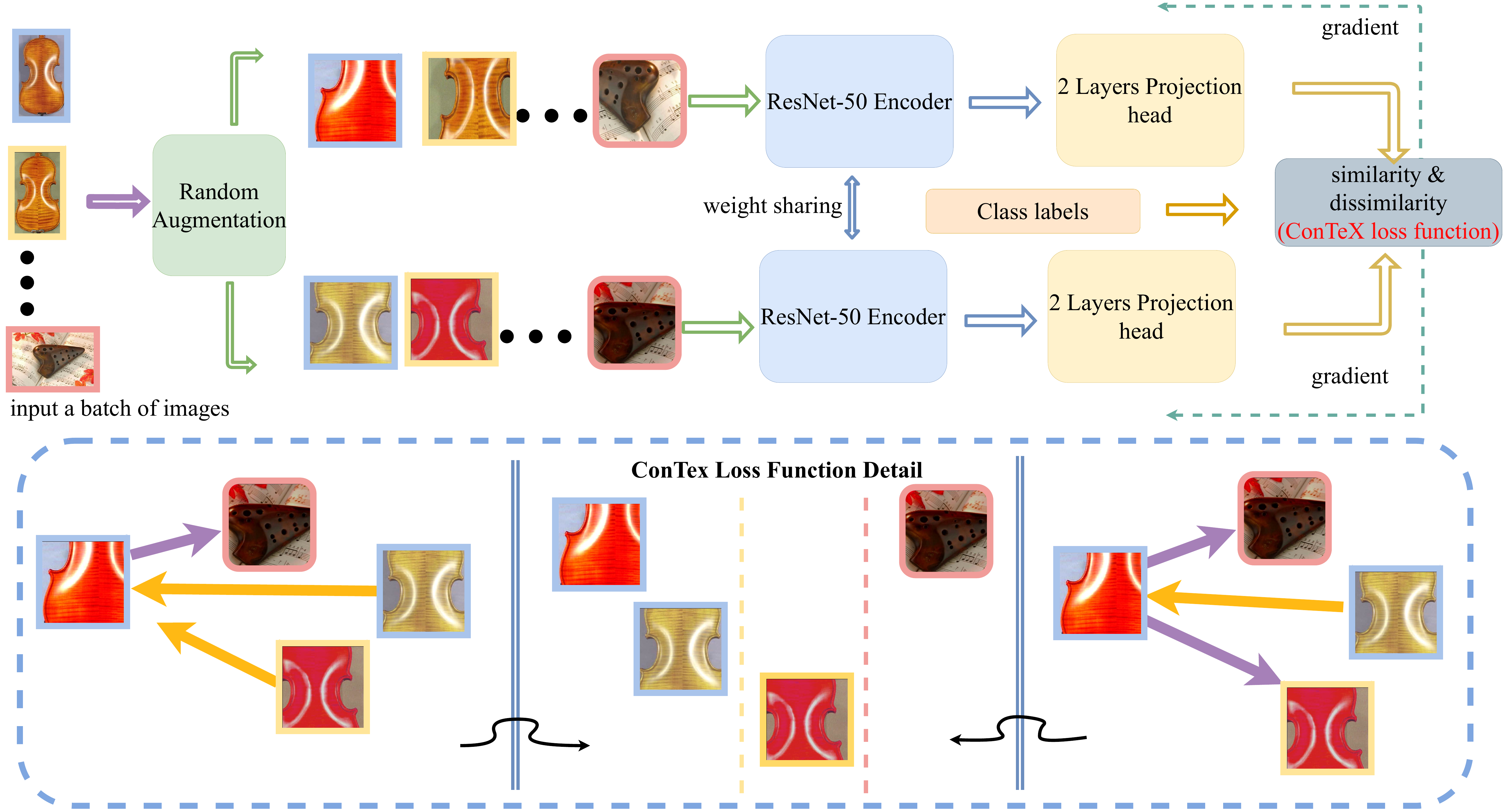}
  \caption{\textcolor{black}{The framework of our stage one training. The images with \textbf{light blue frame} are the \textit{anchor} and its \textit{self positive} (the augmented sample from the same original sample as the \textit{anchor}). The images with \textbf{red frame} are \textit{context negatives} (samples have different labels with \textit{anchor}). The images with \textbf{yellow frame} are \textit{context positives} (samples have the same label with \textit{anchor}). Both the images with the light cyan frame and yellow frame belong to the `Violin' class and the images with the red frame are from the `Ocarina' class. In the linear evaluation and fine-tuning stage, we only use the ResNet-50 encoder from stage one and add a linear fully connected layer as a classifier. The output dimension of ResNet-50 is 2048, and each layer of the projection head block contains 128 neurons.}}
  \label{fig:framework}
\end{figure*}

\subsection{Preliminary}
To elucidate our approach, we extend data samples into the four distinct groups from the original two distinct groups in the existing methods. Fig.~\ref{fig:loss} and~\ref{fig:framework} illustrate the relationships among these four groups.
\begin{itemize}
    \item \textbf{Self Positive}: each image undergoes two augmentations. For any given `anchor' —an augmented sample acting as a reference in the latent space—there's a single `self positive', which is its paired augmentation from the same original image.

    \item \textbf{Self Negatives}: these are all other samples derived from different original images than the anchor.

    \item \textbf{Context Positives}: samples sharing the same context class as the anchor.

    \item \textbf{Context Negatives}: samples from different context classes than the anchor.
\end{itemize}
Our proposed loss function is designed to seamlessly integrate with established contrastive learning frameworks~\cite{SimCLR,supcon}. To facilitate a clear comparison between our proposed ConTeX loss function and the previous works, we adopt similar notations from SupCon~\cite{supcon}. The input data samples, represented as (data, label) pairs, are denoted as $\{x_k,y_k\}$ where $k=1\dots N$. \textcolor{black}{We adopted the data augmentation strategy from \cite{SimCLR}, where each original image is transformed into two distinct views to enhance the model's robustness and generalization capabilities.} Let $i\in I$ where $I \equiv \{1\dots2N\}$ represents the index of augmented samples in a batch of size $N$. For each augmented data pair, we represent as $\{\tilde{x}_j,\tilde{y}_j\}$ where $j=1\dots 2N$. Both the $\tilde{x}_{j-1}$ and $\tilde{x}_j$ are the augmented samples of $x_{(j-1)/2}$, therefore $\tilde{y}_{j-1}$ and $\tilde{y}_j$ have the same labels. 

Mathematically, the augmentation function $t$ generate two different augmented images $x^\sim_q$ and $x^\sim_k$ by:
\begin{equation}
{\tilde{x}_q = t(x), \tilde{x}_k = t'(x),}
\end{equation} 
\textcolor{black}{where $t$ and $t'$ are two independently and randomly generated augmentation methods. The details of these augmentations are provided in Sec.~\ref{sec:exp}.}
Subsequently, these augmented images are encoded using a weight-sharing encoder $E$ to generate representations $h_q$ and $h_k$:
\begin{equation}
{h_q = E(\tilde{x}_q), h_k = E(\tilde{x}_k).}
\end{equation}

We include the projection head in the design to mitigate the risk of the encoder learning irrelevant features under the influence of the contrastive loss function especially when the framework is used to establish similarities between variations of the same image, $x^\sim_q$ and $x^\sim_k$. This projection head absorbs these superfluous features and can be conveniently removed for downstream tasks. In our framework, this weight-sharing projection head is denoted as $Proj$ and is configured as a 2-layer Multilayer Perceptron (MLP) that incorporates a Rectified Linear Unit (ReLU) as the activation function:
\begin{equation}
{z_q = Proj(h_q), z_k = Proj(h_k).}
\end{equation}

In the final step, the objective is to minimize the contrastive loss function $\mathcal{L}$. 

\subsubsection{Self-supervised contrastive loss function}
In most recent self-supervised learning studies, the contrastive loss function is:
\begin{equation}\label{SimCLR loss}
{
\mathcal{L}^{self} = -\sum_{i\in{I}}\log\frac{exp(z_{i}	\cdot z_{P_{s}(i)}/\tau)}{\sum_{a\in{A(i)}}\exp(z_{i} \cdot z_{a}/\tau)}.
}
\end{equation}
In this equation, we use $\cdot$ as the dot product and $\tau$ as temperature. The purpose of this loss function is to approach $max(sim(z_{i}, z_{P_{s}(i)}))$ and $min(\sum_{a\in{A(i)}}sim(z_{i}, z_{a}))$. We use $A(i) \equiv I \textbackslash i$ to represent all the other augmented samples corresponding to index $i$. $P_s(i)$ represents the index of \textit{self positive} to $i$ (augmented from the same image).

\subsubsection{Supervised contrastive loss function}
To leverage the context information in contrastive learning,~\cite{supcon} proposed a supervised contrastive learning loss function without cross-entropy loss:
\begin{equation}\label{SupCon loss}
{
\mathcal{L}^{sup} = \sum_{i\in{I}}\frac{-1}{|P_l(i)|}\sum_{p\in{P_l(i)}}\log\frac{exp(z_{i}\cdot z_{p}/\tau)}{\sum_{a\in{A(i)}}\exp(z_{i}\cdot z_{a}/\tau)},
}
\end{equation}
where $P_l(i) \equiv p \in \{A(i): \tilde{y}_p=\tilde{y}_i\}$ denotes the index of \textit{context positives} with the same label as $i$. This loss function utilizes label information to maximize the similarity between augmented images that share the same label. The supervised loss function aims to achieve $max(sim(z_{i}, z_{P_{l}(i)}))$ using label information while also trying to reach $min(\sum_{a\in{A(i)}}sim(z_{i}, z_{a}))$. It is important to note that the \textit{context positives} are also incorporated in the minimizing term to maintain the stability of the loss function. 

\subsection{ConTeX loss function}
Inspired by the supervised loss function, we proposed a context-enriched contrastive loss function. This loss function comes with two parts. The first part learns from \textit{context positives} with higher sensitivity on contrastive learning. The second part learns from \textit{self positive}.

\begin{equation}\label{loss part 1}
{
\textcolor{black}{
\mathcal{L}_{a} = \sum_{i\in{I}}\frac{-1}{|P_l(i)|}\sum_{p\in{P_l(i)}}\log(\frac{\exp(z_{i}\cdot z_{p}/\tau)}{\sum_{n_2\in{N_l(i)}}\exp(z_{i}\cdot z_{n_2}/\tau)}).
}
}
\end{equation}

In the denominator, $N_l(i) \equiv n_2 \in \{A(i)\textbackslash P_l(i) \}$ represents the index of \textit{context negatives} with labels different from $i$. \textcolor{black}{To fully leverage the potential information from the context, we have chosen to exclude the \textit{context positives} from the minimizing term (the denominator).} This approach allows for more effective utilization of context information in the learning process by directly contrasting the similarity of the current sample with \textit{context positives} and \textit{context negatives}. 
\begin{displaymath}
{max(sim(z_{i}, z_{P_{l}(i)})),\; min(\sum_{n_2\in{N_l(i)}}sim(z_{i}, z_{n_2})).}
\end{displaymath} 

On the other hand, we need to facilitate \textit{self positive} in distinguishing \textit{context positives} to improve the generalizability of our loss function especially in transfer learning.

\begin{equation}\label{loss part 2}
{
\mathcal{L}_{b}  = -\sum_{i\in{I}}\log(1+\frac{\exp(z_{i}\cdot z_{P_s(i)}/\tau)}{\sum_{n_1\in{N_s(i)}}\exp(z_{i}\cdot z_{n_1}/\tau)}),
}
\end{equation}
where $N_s(i)$ is the \textit{self negatives} to $i$ where $N_s(i) \equiv n_1 \in I \textbackslash \{i, P_s(i)\}$.
Inspired by (N+1)-Tuplet Loss~\cite{NT-Xent}, we use the \textit{self positive} as the only positive pair and all the other augmented samples as negative pairs. In latent space, we aimed to keep the \textit{self positive} as the most similar one across the whole batch by

\begin{displaymath}
{max(sim(z_{i}, z_{P_{s}(i)})),\; min(\sum_{n_1\in{N_s(i)}}sim(z_{i}, z_{n_1})).}
\end{displaymath} 

To better reduce the conflict between context-similarity and self-similarity in the combination loss function, we introduce a weight parameter $\lambda$. The ConTeX loss function could be represented as:

\begin{equation}\label{Proposed loss hard}
{
\textcolor{black}{
\begin{aligned}
&\mathcal{L}^{ConTeX} = \\
&\sum_{i\in{I}}\left(\lambda\frac{-1}{|P_l(i)|}\sum_{p\in{P_l(i)}}(\log(\frac{\exp(z_{i}\cdot z_{p}/\tau)}{\sum_{n_2\in{N_l(i)}}\exp(z_{i}\cdot z_{n_2}/\tau)})\right.\\
&\left.-(1-\lambda)\log(1+\frac{\exp(z_{i}\cdot z_{P_s(i)}/\tau)}{\sum_{n_1\in{N_s(i)}}\exp(z_{i}\cdot z_{n_1}/\tau)})\right).
\end{aligned}
}
}
\end{equation}

\subsection{Upper Bound Analysis of the Loss Function}

SimCLR proposed NT-Xent as their contrastive loss function. NT-Xent used one pair of samples as positive and the rest samples from batch size as negative. The advantage is this loss function could benefit from a larger batch size because it would have more negative samples to contrast. However, the efficiency is acceptable but not promising enough because the loss function could arrange all same downstream-label samples as negative and extend the distance on latent space. To address the following issue, we systematically analyzed the upper bound of the NT-Xent loss function and part 1 of our proposed loss function (Eq.~\ref{loss part 1}). We hypothesize that our proposed loss function could benefit from multi-positive pairs and hence gain more information in larger batch size compared to SimCLR.

\begin{lemma}\label{lemma batch}
At least two positive pairs correspond to one representation in a mini-batch exist if the batch size $n$ is larger than the number of label classes $m$.
\end{lemma}
\begin{proof}
If the batch size $n=m$, the worst case is that all labels $L_1,\dots L_m$ exist in one sample. However, if $n=m+1$ one subset in $\{L_1,\dots L_m\}$ must exists more than one sample which means in at least one subset $L_i$ there exists two positive pairs $\{(x_{a_q}, x_{a_k}),(x_{a_q}, x_{b_k})\}$ correspond to a single representation.
\end{proof}

\subsubsection{Upper Bound}
According to Lemma~\ref{lemma batch}, we will analyze the upper bound of our proposed part 1 loss function (context-based loss function) when the mini-batch size is larger than the number of label classes. Based on our loss function Eq.~\ref{loss part 1}, the total loss for all samples in batch size $n$ could be present as:

\begin{equation}\label{fullbatchloss}
\begin{aligned}
\mathcal{L} = &-\frac{1}{n}\sum^n_a\log\frac{\sum_{c\in L_i} \exp({ sim(z_{a}, z_{c})/\tau})}{\sum^{2n}_b \mathds{1}_{[b \notin L_i]}\exp({sim(z_{a}, z_{b})/\tau})}\\
&\ \ a\in L_i, i \in m,
\end{aligned}
\end{equation} 
where $L_i$ represents one of the label sets in $m$ classes that always contains the index of the same label samples as $a$. Based on the logarithmic rules we can rewrite Eq.~\ref{fullbatchloss} to 
\begin{equation}
\begin{aligned}
\mathcal{L} = & -\frac{1}{n}\sum^n_a\log(\sum_{c\in L_i} \exp({ sim(z_{a}, z_{c})/\tau}))\\
& +\frac{1}{n}\sum^n_a\log(\sum^{2n}_b \mathds{1}_{[b \notin L_i]}\exp({sim(z_{a}, z_{b})/\tau}))\\ 
&\ \ a\in L_i, i \in m,
\end{aligned}
\end{equation} 
where $x_{a,b}$ could be denoted as $sim(z_a, z_b)/\tau$. Then we can rewrite the loss function as 
\begin{equation}
\begin{aligned}
\mathcal{L} = & -\frac{1}{n}\sum^n_a\log(\sum_{c\in L_i} \exp({x_{a,c}}))\\
& +\frac{1}{n}\sum^n_a\log(\sum^{2n}_b \mathds{1}_{[b \notin L_i]}\exp({x_{a,b}})) \ \ a\in L_i, i \in m.
\end{aligned}
\end{equation}

Since there exists at least one positive pair of sample, and \textcolor{black}{only $N$ negative pairs} ($\mathds{1}_{[b \notin L_i]}$) exist in the second term in the loss function
\begin{equation}
\begin{aligned}
\log(\sum^{2n}_b \mathds{1}_{[b \notin L_i]}\exp({x_{a,b}})) < LSE(x_{a,1},\dots, x_{a,N})
\end{aligned}
\end{equation} 

\textcolor{black}{LogSumExp(LSE) term can be define as:}
\begin{equation}\label{LSE}
\begin{aligned}
\textcolor{black}{LSE(x_1,\dots, x_n) = \log(\exp(x_1)+\dots+ \exp(x_n)).}
\end{aligned}
\end{equation} 

Based on the properties of the LogSumExp function
\begin{equation}
\begin{aligned}
max\{x_1, \dots,x_n\}&\leq LSE(x_1,\dots, x_n)\\
&\leq max\{x_1, \dots,x_n\}+\log(n),
\end{aligned}
\end{equation}

we can simplify the upper bound of our proposed loss function by
\begin{equation}\label{long proof}
\begin{aligned}
\mathcal{L} < & -\frac{1}{n}\sum^n_a{LSE(x_{a,L_{i_1}},\dots, x_{a,L_{i_l}})}\\
& +\frac{1}{n}\sum^n_a{LSE(x_{a,1},\dots, x_{a,N})} \ \ a\in L_i, i \in m, \\
< & -\frac{1}{n}\sum^n_a(max(x_{a,L_{i_1}},\dots, x_{a,L_{i_l}})+\log(|L_i|))\\
& +\frac{1}{n}\sum^n_a(max(x_{a,1},\dots, x_{a,N})+\log(N)) \ \ a\in L_i, i \in m, \\
= & -\frac{1}{n}\sum^n_a(max(x_{a,L_{i_1}},\dots, x_{a,L_{i_l}}))-\log(|L_i|)\\
& +\frac{1}{n}\sum^n_a(max(x_{a,1},\dots, x_{a,N}))+\log(N) \ \ a\in L_i, i \in m,
\end{aligned}
\end{equation} 
where $l = |L_i|$ is the number of postive pairs related to $x_a$. Based on Agren's study~\cite{Nt-Xent_Upper_Bound}, the upper bound of the NT-Xent loss function could be represent as
\begin{equation}
\begin{aligned}
\mathcal{L}^{NT-Xent}=&-\frac{1}{n}\sum_{i,j \in MB}x_{i,j}+\log(N)\\
&+\frac{1}{n}\sum^n_i max(x_{i,1},\dots,x_{i,N}),
\end{aligned}
\end{equation}
where $MB$ refers to the current mini-batch of data and $x_{i,j}$ represents the only positive pair in the NT-Xent. We will compare our loss function expression with the NT-Xent term by term. In the first term of Eq.~\ref{long proof}, $max(x_{a,L_{i_1}},\dots, x_{a,L_{i_l}})$ already included the only positive pair from NT-Xent $x_{i,j}$. Therefore,
\begin{equation}\label{term1}
\begin{aligned}
max(x_{a,L_{i_1}},&\dots, x_{a,L_{i_l}}) \geq x_{i,j}, \\
-\frac{1}{n}\sum^n_a(max(x_{a,L_{i_1}},&\dots, x_{a,L_{i_l}})) \leq -\frac{1}{n}\sum_{i,j \in MB}x_{i,j}.
\end{aligned}
\end{equation}
Furthermore, since the number of positive pairs in the set $L_i$ must be positive. Based on Lemma~\ref{lemma batch}, when there exists more than one positive pair, $|L_i|>0$ and
\begin{equation}\label{term2}
\begin{aligned}
-\log(|L_i|) +\frac{1}{n}\sum^n_a(max(x_{a,1},&\dots, x_{a,N}))+\log(N) \\
\leq \log(N)+ \frac{1}{n}\sum^n_i max(x_{i,1},&\dots,x_{i,N}).
\end{aligned}
\end{equation}

Based on the two parts of proof (Eq.~\ref{term1}, \ref{term2}), our loss function has a lower upper bound compared to the NT-Xent if and only if there exists more than one positive sample pair($|L_i|$) in the current mini batch. More information is gained in contrast if there are more positive pairs in a mini-batch. Based on Eq.~\ref{term1}, if there exists a pair $x_{a,L_i}$ larger than the original positive pair $x_{i,j}$ in the NT-Xent, the upper bound could be further reduced.
\begin{color}{black}
\subsection{Theoretical Justifications}
This part will discuss benefits from the additional margin in ConTeX loss function. The gradient of the ConTeX loss \textit{part a} with respect to $z_i$ is given by:

\begin{equation}
\begin{aligned}
&\frac{\partial \mathcal{L}^{ConTeX}_{a}}{\partial z_i} = \\
&\frac{-1}{|P(i)|} \sum_{p \in P_{l}(i)} \frac{\partial}{\partial z_i} 
\left\{ \frac{z_i \cdot z_p}{\tau} - \log \sum_{n_2 \in N_{l}(i)} \exp(z_i \cdot z_{n_2} / \tau) \right\}
\end{aligned}
\end{equation}

\begin{equation}
= \frac{-1}{\tau |P_{l}(i)|} \sum_{p \in P_{l}(i)} \left\{ z_p - \frac{\sum_{n_2 \in N_{l}(i)} z_{n_2} \exp(z_i \cdot z_{n_2} / \tau)}{\sum_{n_2 \in N_{l}(i)} \exp(z_i \cdot z_{n_2} / \tau)} \right\}.
\end{equation}

Rewriting the expression in terms of $X_{in}$:

\begin{equation}
= \frac{-1}{\tau |P_{l}(i)|} \sum_{p \in P_{l}(i)} \left\{ z_p - \sum_{n_2 \in N_{l}(i)} z_{n_2} X_{in_2} \right\}
\end{equation}
\begin{equation}
\label{eq:g_contex_p1}
= -\frac{1}{\tau} \left\{ \bar{z}_{p_l} - \sum_{n_2 \in N_{l}(i)} z_{n_2} X_{in_2} \right\}
\end{equation}
where
\begin{equation}
\quad X_{in} \equiv \frac{\exp(z_i \cdot z_{n} / \tau)}{\sum_{n \in N(i)} \exp(z_i \cdot z_{n} / \tau)}
\end{equation}
 represents the likelihood of the negative sample \( z_n \) given the anchor \( z_i \) relative to all other negative samples. In our specific formulation, this represents the probability of selecting the current context-specific negative sample over all other context-specific negatives. Additionally, $\bar{z}_{p_l}$ signifies the average of all \textit{context positive} representations.

The gradient of the ConTeX loss \textit{part b} with respect to $z_i$ is given by:

\begin{equation}
\begin{aligned}
&\frac{\partial \mathcal{L}^{ConTeX}_{b}}{\partial z_i} = -\frac{\partial}{\partial z_i} 
\left\{ \frac{z_i \cdot z_{P_s(i)}}{\tau} - \log \sum_{a \in A(i)} \exp(z_i \cdot z_a / \tau) \right\}
\end{aligned}
\end{equation}
\begin{equation}
\label{eq:g_contex_p2}
= -\frac{1}{\tau} \left\{ z_{p_s} - z_{P_s(i)}P_{ips} - \sum_{n_1 \in N_{s}(i)} z_{n_1} P_{in_1} \right\}
\end{equation}
where
\begin{equation}
\quad P_{in} \equiv \frac{\exp(z_i \cdot z_n / \tau)}{\sum_{a \in A(i)} \exp(z_i \cdot z_a / \tau)}
\end{equation}
\( P_{in} \) represents the likelihood of the negative sample \( z_n \) given the anchor \( z_i \) relative to all other samples in the set \( A(i) \). In Eq.~\ref{eq:g_contex_p2}, $P_{ips}$ and $P_{in_1}$ represents the likelihood of \textit{self positive} and \textit{self negative} respectively. To simplify the calculation, we set $\lambda = 0.5$, combining two part together we have

\begin{equation}
\begin{aligned}
&\frac{\partial \mathcal{L}^{ConTeX}}{\partial z_i} = 
\frac{1}{\tau} \left\{ z_{P_s(i)}P_{ips} + \sum_{n_1 \in N_{s}(i)} z_{n_1} P_{in_1} \right.\\
&\left.+ \sum_{n_2 \in N_{l}(i)} z_{n_2} X_{in_2}
- z_{p_s} - \bar{z}_{p_l} \right\}.
\end{aligned}
\end{equation}

Given the gradient of SupCon\cite{supcon} loss function is 

\begin{equation}
\begin{aligned}
&\frac{\partial \mathcal{L}^{SupCon}}{\partial z_i} = 
\frac{1}{\tau} \left\{ z_{P_l(i)}P_{ipl} + \sum_{n_2 \in N_{l}(i)} z_{n_2} P_{in_2} - \bar{z}_{p_l} \right\}.
\end{aligned}
\end{equation}

ConTex loss explicitly incorporates different types of interactions within its loss calculation compared to SupCon Loss. Therefore, it potentially establishes a more significant margin between classes by emphasizing both intra-class compactness and inter-class separability through its dual consideration of self and contextual relationships. Specifically, the ConTeX loss includes not only context-based positives and negatives (similar to SupCon) but also integrates self-positives and self-negatives. This addition enhances the contrast between self-positives and self-negatives, reinforcing negative constraints within each class that likely enhance the discriminative power of the model. Essentially, this creates a richness of more learnable features, allowing the model to learn more content and knowledge. By managing a variety of sample relationships and ensuring that the model does not overly focus on any single type of sample interaction, ConTeX is likely more robust against overfitting, adapting better to new, unseen, and biased data compared to SupCon. Experimentation and empirical validation from Table.~\ref{BMNIST} and Table.~\ref{transfer_imagenet32} also substantiate these theoretical advantages in specific tasks.
\end{color}
\section{Experiments}
\label{sec:exp}
\subsection{Experimental Setup}
In this section, we detail our experimental environment, which includes the data sets used, the methods compared to and the specifics of the training framework.
Our loss function, as presented in Eq.~\ref{Proposed loss hard}, is evaluated against two categories of methods: non-bias methods and bias methods.

\paragraph{General Classification Experiments}
For general image classification, we benchmarked our loss function against established methods such as SimCLR~\cite{SimCLR}, Max Margin~\cite{LargeMargin}, Cross Entropy, and SupCon~\cite{supcon}. The datasets utilized for this purpose include CIFAR-10, CIFAR-100~\cite{cifar}, and ImageNet~\cite{imagenet}. Furthermore, we assessed transfer learning capabilities on Caltech-101~\cite{caltech101} and Caltech-256~\cite{caltech256}, particularly contrasting our method with the SupCon loss.

\paragraph{Fairness-oriented Methods}
When addressing fairness enhancement, our evaluations were conducted on datasets like BiasedMNIST~\cite{bmnist}, UTKFace~\cite{utkface}, and CelebA~\cite{cleleba}, adopting settings akin to the work by Hong et al.~\cite{hong_unbiased_2021}. In scenarios necessitating bias labels, we drew comparisons with methods such as Vanilla network, LNL~\cite{lnl}, DI~\cite{di}, EnD~\cite{end}, BiasBal~\cite{hong_unbiased_2021}, and BiasCon\cite{hong_unbiased_2021}. Alternatively, in label-independent contexts, we compared against Vanilla network, LM~\cite{LM}, RUBi~\cite{rubi}, ReBias~\cite{bahng_learning_2020}, LfF~\cite{lff}, and SoftCon~\cite{hong_unbiased_2021}.

\paragraph{Training Framework}
For training, each input batch comprises images subjected to dual augmentations. These doubly-augmented images are processed through our ResNet encoder that is initialized with random weights~\cite{resnet}. This encoder yields embeddings that are subsequently forwarded to a projection block, producing two batches of 128-dimensional latent space features. This projection block is discarded during linear evaluation or fine-tuning phases. 

\paragraph{Augmentation Mechanism}
\textcolor{black}{In our experiment environment, we employed a comprehensive data augmentation pipeline to enhance the training process for our contrastive learning models. The pipeline included randomly resizing and cropping images to varying scales between 20\% and 100\% of the original size, followed by a random horizontal flip to introduce mirrored variations. We applied color jitter with a probability of 80\%, adjusting brightness, contrast, saturation, and hue to create diverse visual conditions. Additionally, a 20\% probability of converting images to grayscale was included to further diversify the dataset. Finally, the images were converted to tensors and normalized, ensuring consistency and readiness for model input. These augmentations collectively improved the robustness and generalization of our models by providing varied and enriched training data.}

Our proposed ConTeX loss function computes the loss based on the latent space features from both batches, after which a back-propagation step updates the model weights. Post-training, the encoder layers are fixed, and a linear classification layer is trained using the ground-truth labels to evaluate our framework's performance in image classification tasks.

\subsection{Performance comparison}
\begin{table}[h]
\centering
\begin{threeparttable}
  \caption{Top-1 linear evaluation accuracy on ResNet-50~\cite{resnet}. Results with \textbf{*} are the accuracy that we re-implemented.}
  \label{top 1}
  \centering
  \begin{tabular}{cccc}
    \toprule
    
    Methods &  CIFAR10 & CIFAR100 & ImageNet\\
    \hline
    SimCLR~\cite{SimCLR} &  93.6 & 70.7 & 70.2\\
    Max Margin~\cite{LargeMargin} & 92.4 & 70.5 & 78.0\\
    Cross Entropy & 95.0* & 72.3* & 77.8*\\
    SupCon~\cite{supcon} &95.3*&74.8*&78.1*\\
    Ours & \textbf{95.9}&\textbf{75.8}&\textbf{78.4}\\
    
    \bottomrule
  \end{tabular}
  \end{threeparttable}
\end{table}
\subsubsection{Performance on ImageNet, CIFAR10, CIFAR100}
%Performance on imagenet, cifar10, cifar100
Table.~\ref{top 1} reflects we have compatible performance on image classification compared to SupCon~\cite{supcon}. Our proposed loss function outperformed max-margin, and cross-entropy on CIFAR10, CIFAR100 and  ImageNet ILSVRC-2012~\cite{imagenet2012} datasets. We pretrained the network and used linear evaluation to get the top-1 accuracy. 

\begin{figure}[h]
  \centering
  \includegraphics[width=\columnwidth]{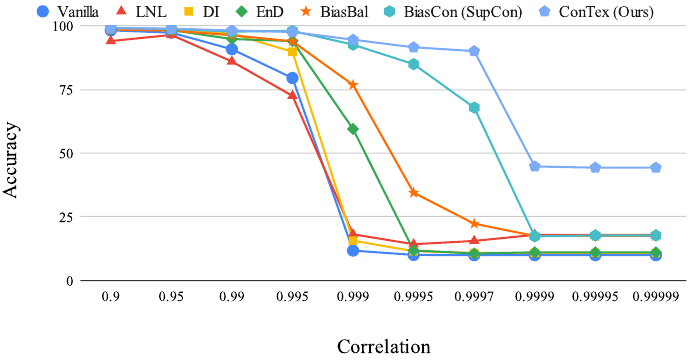}
  \caption{Comparison of unbiased accuracy on the BiasedMNIST~\cite{bmnist} dataset with high target-bias correlations.}
  \label{fig:bias_label_BMNIST}
\end{figure}

\begin{table*}[h]
\centering
\begin{threeparttable}
    
  \caption{Unbiased accuracy (std) on the BiasedMNIST~\cite{bmnist} dataset with high target-bias correlations on training stage. Results with \textbf{*} are the accuracy that we re-implemented. Others are based on study~\cite{hong_unbiased_2021}. Results not significantly worse than the best (p > 0.05, permutation test) are shown in bold.}
  \label{BMNIST}
  \centering
  \begin{tabular}{cccccccc}
    \toprule
    Correlation$^a$& Vanilla& LNL~\cite{lnl}& DI~\cite{di}& EnD~\cite{end} & BiasBal~\cite{hong_unbiased_2021} & BiasCon~\cite{hong_unbiased_2021} & \textbf{Ours} \\
    \midrule
    0.9999 & *10.0(0.00) & *18.0(0.00) & *10.9(0.02) & *11.1(0.03) & *17.6(0.30) & *17.7(0.21) & \textbf{44.7(0.42)}\\
    0.9997 & *10.0(0.00) & *15.5(0.14) & *10.7(0.22) & *10.6(0.07) & *22.4(0.00) & *67.5(0.43) & \textbf{90.4(0.67)}\\
    0.9995 & *10.1(0.03) & *14.5(0.28) & *11.6(0.05) & *11.7(0.19) & *34.7(0.25) & *85.4(0.47) & \textbf{92.17(0.71)}\\
    0.999  & 11.8(0.7) & 18.2(1.2) & 15.7(1.2) &59.5(2.3) & 76.8(1.6) & *92.8(0.27) & \textbf{95.0(0.52)}\\
    0.995  & 79.5(0.1) & 72.5(0.9) & 89.8(2.0) & 94.0(0.6) & 93.9(0.1) & \textbf{*97.8(0.14)} & \textbf{97.5(0.03)}\\
    0.99   & 90.8(0.3) & 86.0(0.2) & 96.9(0.1) & 94.8(0.3) & 96.3(0.2) & \textbf{*97.7(0.20)} & \textbf{98.1(0.10)}\\
    0.95   & 97.3(0.2) & 96.4(0.1) & \textbf{98.6(0.1)} & \textbf{98.3(0.1)} & 98.1(0.0) & \textbf{*98.7(0.03)} & \textbf{98.8(0.11)}\\
    0.9    & 98.2(0.1) & 94.0(0.3) & \textbf{98.8(0.1)} & \textbf{98.7(0.0)} & 98.5(0.1) & \textbf{*99.1(0.09)} & \textbf{98.9(0.10)}\\
    \bottomrule
  \end{tabular}
  \begin{tablenotes}
   \item[$a$] The correlation between the background color and data features, it becomes more difficult to distinguish between them when the correlation is higher.
  \end{tablenotes}
\end{threeparttable}
\end{table*}

\subsubsection{Bias Mitigation}
\textcolor{black}{In traditional supervised contrastive learning approaches, such as SupCon, the use of augmentations from the same class as positives can inadvertently lead to the creation of shortcuts (Fig.~\ref{fig:mislead_aug}). This occurs because the model may start to rely on simple, potentially spurious patterns that are common among augmentations of different original images within the same class, rather than learning a more robust and general representation. For instance, if certain features like texture or background are more prevalent in the augmentations of one class, the model might prioritize these easy-to-learn features over more complex but crucial features that genuinely define the class. This reliance can degrade the model’s generalization ability as it may perform well on similar seen data but poorly on unseen, varied data from the same class.}

Driven by this imperative, we seek to rigorously assess the proficiency of our ConTeX method in the context of bias mitigation. The overarching concern of bias in machine learning models propels us to ascertain that ConTeX not only delivers superior performance but also adheres to ethical considerations by actively reducing biases.

In pursuit of this, we have meticulously crafted and conducted three separate experiments. The primary aim is to evaluate ConTeX's prowess in counteracting biases, setting its performance against contemporary state-of-the-art techniques. To ensure a holistic assessment, we have chosen three distinct datasets, each renowned for its intrinsic biases: BiasedMNIST~\cite{bmnist}, UTKFace~\cite{utkface}, and CelebA~\cite{cleleba}. These experiments are designed to underscore the resilience and ethical integrity of the ConTeX methodology when applied to real-world challenges.

To measure the effectiveness of our approach in debiasing tasks, we adopt the unbiased accuracy metric~\cite{bahng_learning_2020}. In scenarios where bias labels are available, our model is trained by leveraging these categorical bias labels. Conversely, in experiments devoid of explicit bias labels, we operate under the assumption that the model possesses inherent knowledge of these labels. In such cases, we exclusively utilize the target label to train the model on biased data, ensuring its adaptability and versatility.

\paragraph{Evaluating Bias Label Performance on BiasedMNIST}

BiasedMNIST~\cite{bmnist} is a modified version of the MNIST dataset, characterized by a background color that exhibits a strong correlation with the data sample. Drawing parallels with the study in~\cite{hong_unbiased_2021}, we employed correlation values, denoted as $\rho$, ranging from ${0.9, 0.95, 0.99, 0.995, 0.999, 0.9995, 0.9997, 0.9999}$. These values signify the degree of association between the background hue and the inherent data attributes.

In Table.~\ref{BMNIST} and Fig.~\ref{fig:bias_label_BMNIST}, we juxtaposed our unbiased accuracy results against contemporary methodologies that also harnessed the bias label during the training phase. A discernible observation is that our approach consistently outperforms across varying correlation metrics. As the correlation between the background color and data features strengthens, it becomes more difficult to distinguish between them during training. Despite this challenge, our method's performance advantage over other techniques becomes even more evident. It is noteworthy that even at a stringent correlation of $\rho = 0.9997$ (where merely 2 out of 5000 training samples lack a correlated background), our technique sustains an accuracy exceeding $90\%$.

\paragraph{Evaluating Bias Label Performance on Real-World Datasets}

\begin{table*}[h]
  \caption{Unbiased accuracy (std) on the UTKFace dataset~~\cite{utkface}. The best results are shown in bold.}
  \label{table:utkface}
  \centering
  \begin{tabular}{ccccccc}
    \toprule
    Bias& Vanilla    & LNL~\cite{lnl}& DI~\cite{di}  & EnD~\cite{end} & BiasCon~\cite{hong_unbiased_2021}  & \textbf{Ours} \\
    \midrule
    Race & 87.4 (0.15) & 87.0 (0.15) & 88.7 (0.50) & 87.6 (0.29)  & 89.9 (0.69)                   & \textbf{90.8 (0.51)}\\
    \midrule
    Age & 72.0 (0.30) & 71.94 (0.51) & 77.0 (1.20) & 71.8 (0.14)   &  75.84 (0.12)                & \textbf{79.0 (1.07)}\\
    \bottomrule
  \end{tabular}
\end{table*}
\begin{table*}[h]
  \caption{Unbiased accuracy (std) on the CelebA dataset~\cite{cleleba}. The best results are shown in bold.}
  \label{table:celeba}
  \centering
  \begin{tabular}{ccccccc}
    \toprule
    Target& Vanilla    & LNL~\cite{lnl}& DI~\cite{di}  & EnD~\cite{end} & BiasCon~\cite{hong_unbiased_2021}  & \textbf{Ours} \\
    \midrule
    Blonde & 80.8 (3.18) & 80.6 (3.53) & 91.0 (0.88) & 80.2 (2.63) & 90.5 (0.35)   & \textbf{91.0 (0.94)}\\
    \midrule
    Makeup & 77.6 (0.93) & 76.6 (0.86) & 76.7 (2.54) & 76.9 (1.99)  &  76.3 (0.75)    & \textbf{83.3 (1.19)}\\
    \bottomrule
  \end{tabular}
\end{table*}

We extended our evaluation to encompass real-world image datasets, specifically UTKFace~\cite{utkface} and CelebA~\cite{cleleba}. For the UTKFace dataset, our experimental design mirrored the approach outlined in~\cite{hong_unbiased_2021}. Here, we trained the model on a binary classification task using \textbf{Gender} as the label, while \textbf{Age} and \textbf{Race} served as bias attributes.

Turning our attention to the CelebA dataset, we set up binary classification tasks with \textbf{HeavyMakeup} and \textbf{BlondHair} as distinct target labels. In both cases, the \textbf{Male} attribute was employed as the bias, aligning with the methodology of~\cite{lff}. For both datasets, we maintained a correlation value of 0.9 between the target and bias labels.

As evidenced in Table.~\ref{table:utkface}, our method consistently delivered robust performance across varied bias attributes. This trend of superior results is further corroborated by the data in Table.~\ref{table:celeba}, which showcases our method's consistent edge across different target tasks. It is worth noting that while BiasCon~\cite{hong_unbiased_2021} relies on SupCon~\cite{supcon} as its foundational loss function, the distinguishing factor in our approach is our novel loss function. This distinction underscores the significant contribution and efficacy of our proposed loss function.

\begin{table*}[h]
\centering
\begin{threeparttable}
  \caption{Unbiased accuracy (std) on the BiasedMNIST dataset without bias labels on training stage. The results with \textbf{ *} are the accuracy that we reimplemented. Others are based on study~\cite{hong_unbiased_2021}. Results not significantly worse than the best (p > 0.05, permutation test) are shown in bold.}
  \label{BMNIST_soft}
  \centering
  \begin{tabular}{m{1cm}ccccccc}
    \toprule
    Correlation$^a$& Vanilla   & LM~\cite{LM} & RUBi~\cite{rubi}& ReBias~\cite{bahng_learning_2020}& LfF~\cite{lff} & SoftCon~\cite{hong_unbiased_2021}  &  \textbf{Ours} \\
    \midrule
    0.999 & 11.8(1.1) & 10.5(0.6) & 10.6(0.5) & 26.5(1.4)  & 15.3(2.9) & *65.6(2.00) & \textbf{74.0(1.32)}\\
    0.997  & 57.2(0.9) & 56.0(4.3)  & 49.6(1.5) &65.8(0.3) & 63.7(20.3) & *87.2(0.60) & \textbf{93.1(0.63)}\\
    0.995  & 74.5(1.4) & 80.9(0.9) & 71.8(0.5) & 75.4(1.0) & 90.3(1.4)& *93.38(0.40) & \textbf{95.8(0.97)}\\
    0.99   & 88.9(0.2) & 91.5(0.4) & 85.9(0.1) & 88.4(0.6) & 95.1(0.1) & *96.6(1.37) & \textbf{97.2(0.25)}\\
    0.95   & 97.1(0.0) & 93.6(0.5) & 96.6(0.1) & 97.0(0.0) & 97.7(0.2) & \textbf{*98.7(0.56)} & \textbf{98.8(0.08)}\\
    0.9    & 98.2(0.1) & 89.5(0.7) & 97.8(0.1) & 98.1(0.1) & 96.1(1.1) & \textbf{*99.0(0.93)} & \textbf{99.0(0.25)}\\
    \bottomrule
  \end{tabular}
  \begin{tablenotes}
   \item[$a$] The correlation between the background color and data features, it becomes more difficult to distinguish between them when the correlation is higher.
  \end{tablenotes}
  \end{threeparttable}
\end{table*}
\paragraph{Evaluating without bias label on BiasedMNIST}
In many real-world scenarios, bias labels might not be available or not in categories. Models that can perform well without explicit bias labels are more versatile and can be applied in a broader range of situations. In order to evaluate the generalization of our proposed loss function, we trained the data from BiasedMNIST with different correlations. To assess the generalizability of our proposed loss function, we subjected the BiasedMNIST dataset to training across varying correlation levels, notably without incorporating explicit bias category data. The results, as delineated in Table.~\ref{BMNIST_soft}, are telling. Our method consistently outperforms the established baselines across all correlation levels. Notably, even in high correlation scenarios where biases are most pronounced and challenging to mitigate, our approach demonstrates superior accuracy. This underscores the efficacy and resilience of our proposed method, highlighting its potential as a benchmark solution for bias mitigation, especially in contexts where bias labels are unavailable.

\paragraph{Implications of ConTeX Loss on Model Generalization and Fairness}

\textcolor{black}{The ConTeX loss mitigates these issues by incorporating both self-positives and a broader array of negatives, including self-negatives and context-specific negatives (detailed in Fig.~\ref{fig:framework}). This setup compels the model to differentiate not just between different classes but also within the same class across different contexts and manifestations. By ensuring that self-positives (augmentations from the same original image) are closer in representation than context-positives (augmentations from different original images within the same class), ConTeX pushes the model to learn more intrinsic and discriminative features of each sample rather than superficial or common ones.}

\textcolor{black}{This nuanced approach to learning representations helps in reducing reliance on shortcuts, thereby enhancing the model’s ability to generalize across varied instances of a class. It also promotes fairness by ensuring that the model does not disproportionately learn to recognize features from dominant subgroups within the class, thus better handling intra-class variations.}

\begin{table}[h]
  \caption{Evaluation of NT-Xent loss on networks pretrained with labels. The results highlight the adverse effects of label bias, where a larger loss value indicates a more pronounced impact.}
  \label{ntxent pretrained eval}
  \centering
  \begin{tabular}{ccccc}
    \toprule
    Pretrained Dataset  &\multicolumn{2}{c}{SupCon~\cite{supcon}}&\multicolumn{2}{c}{Ours} \\

     &train&test&train&test\\
     \midrule
    CIFAR10 & 8.953 &6.05 &\textbf{6.400}&\textbf{3.660}\\
    CIFAR100& 2.435 &4.17&\textbf{1.362}&\textbf{2.374}\\
    ImageNet& \textbf{1.022 }&1.182&1.090&\textbf{1.155}\\
    \bottomrule
  \end{tabular}
\end{table}

\subsection{Misleading by the labels}
%the small dataset transfer learning

\begin{figure*}[th]
\centering
\includegraphics[width=.25\textwidth]{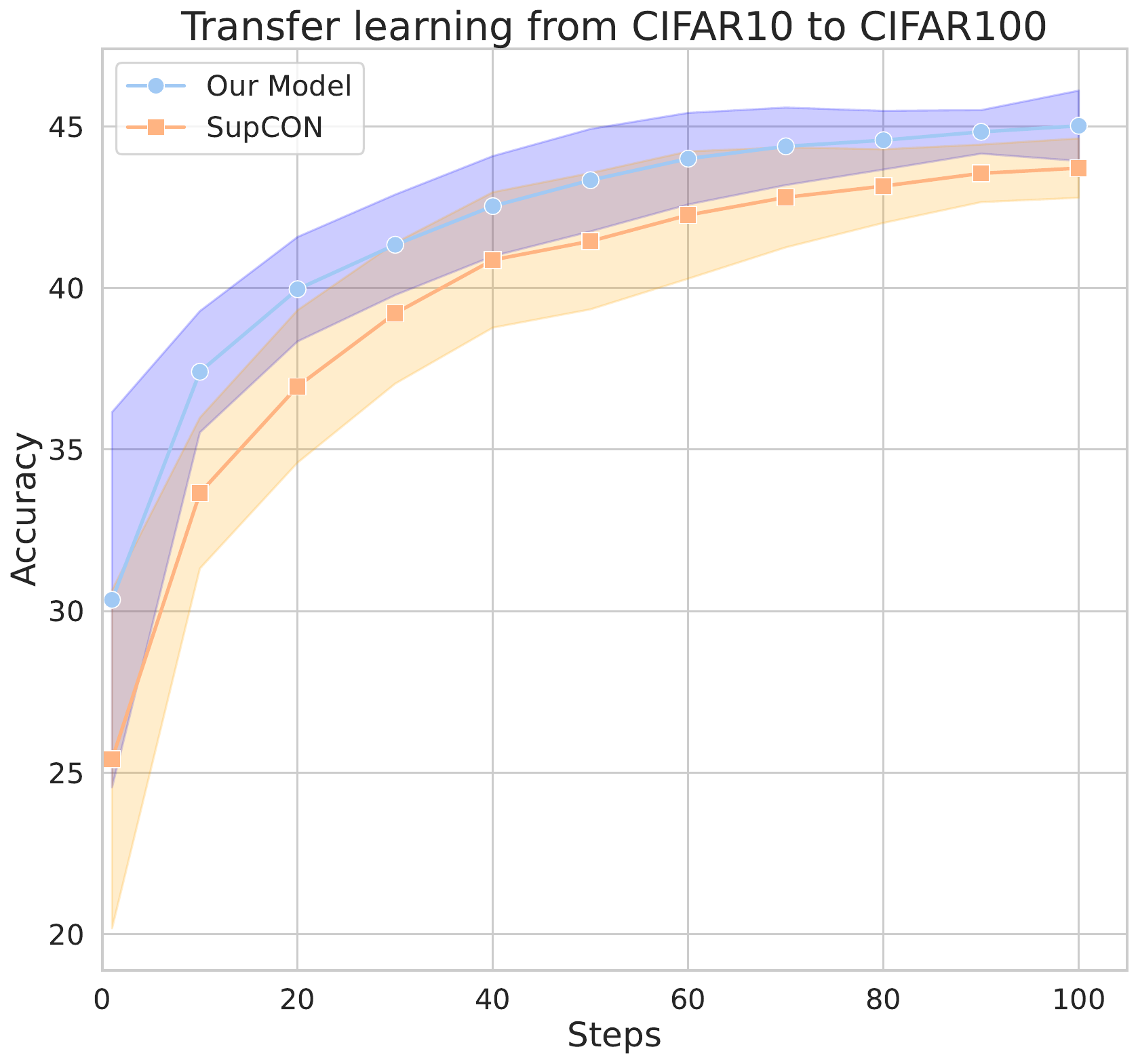}\hfill
\includegraphics[width=.25\textwidth]{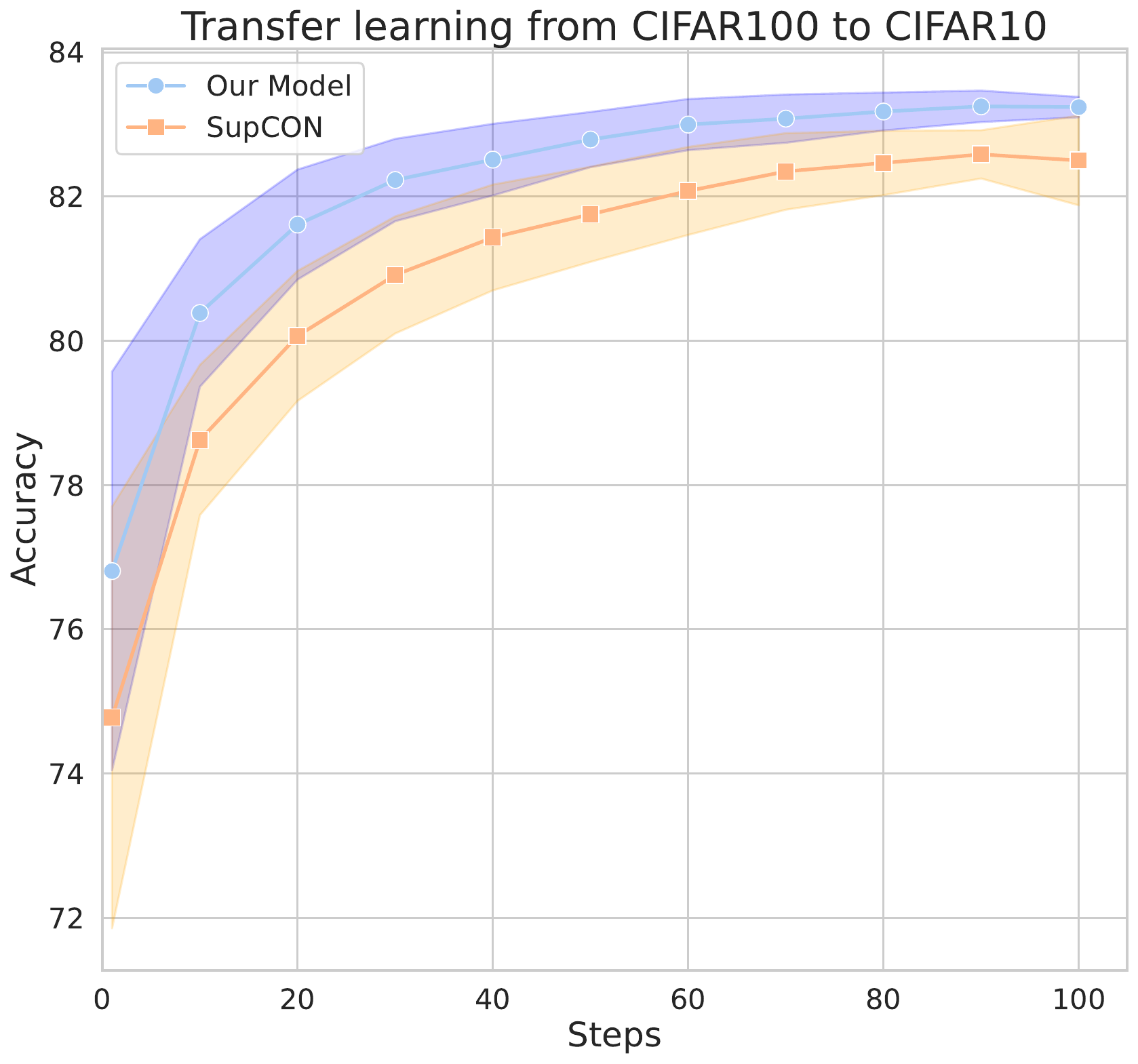}\hfill
\includegraphics[width=.25\textwidth]{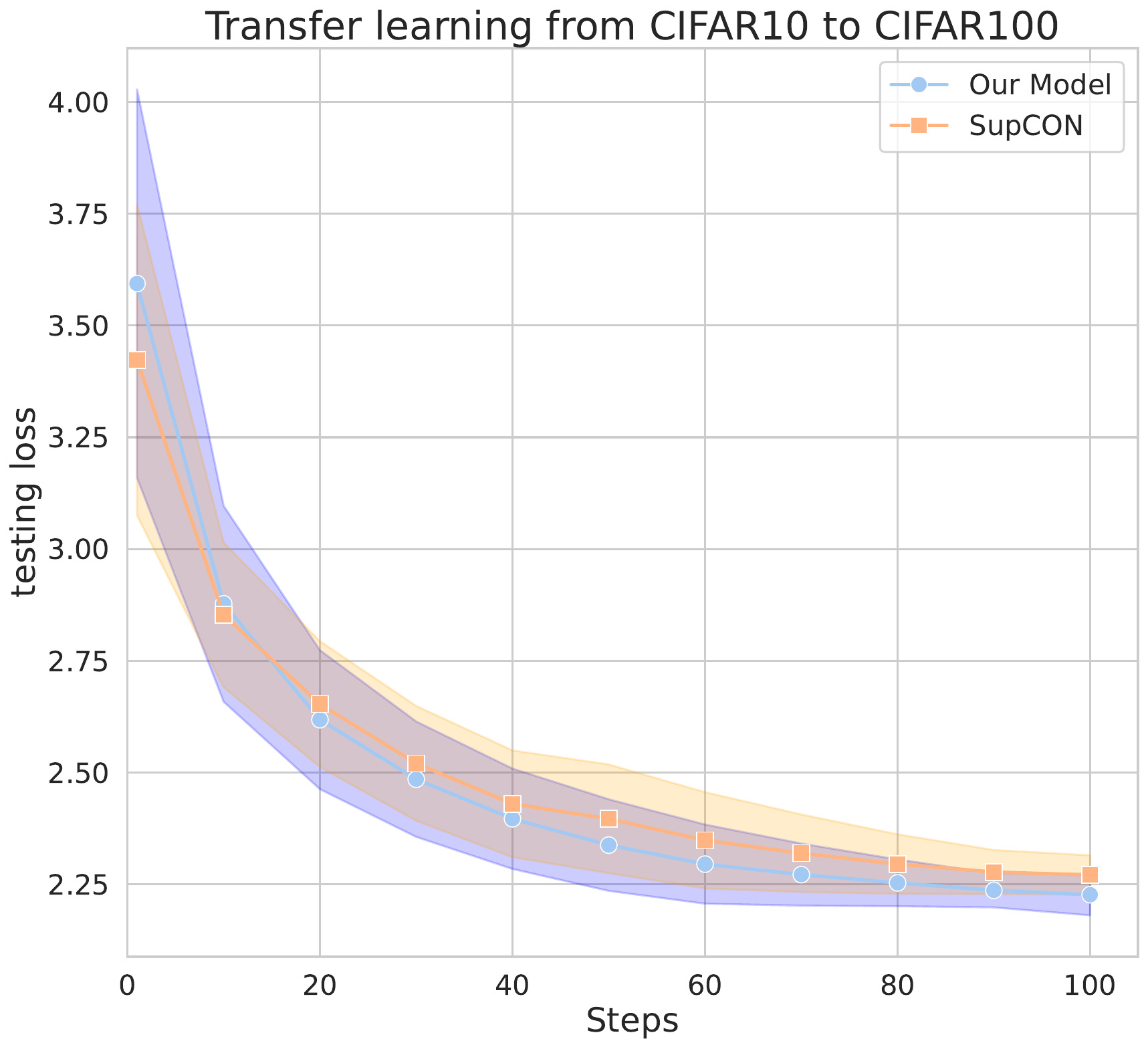}\hfill
\includegraphics[width=.25\textwidth]{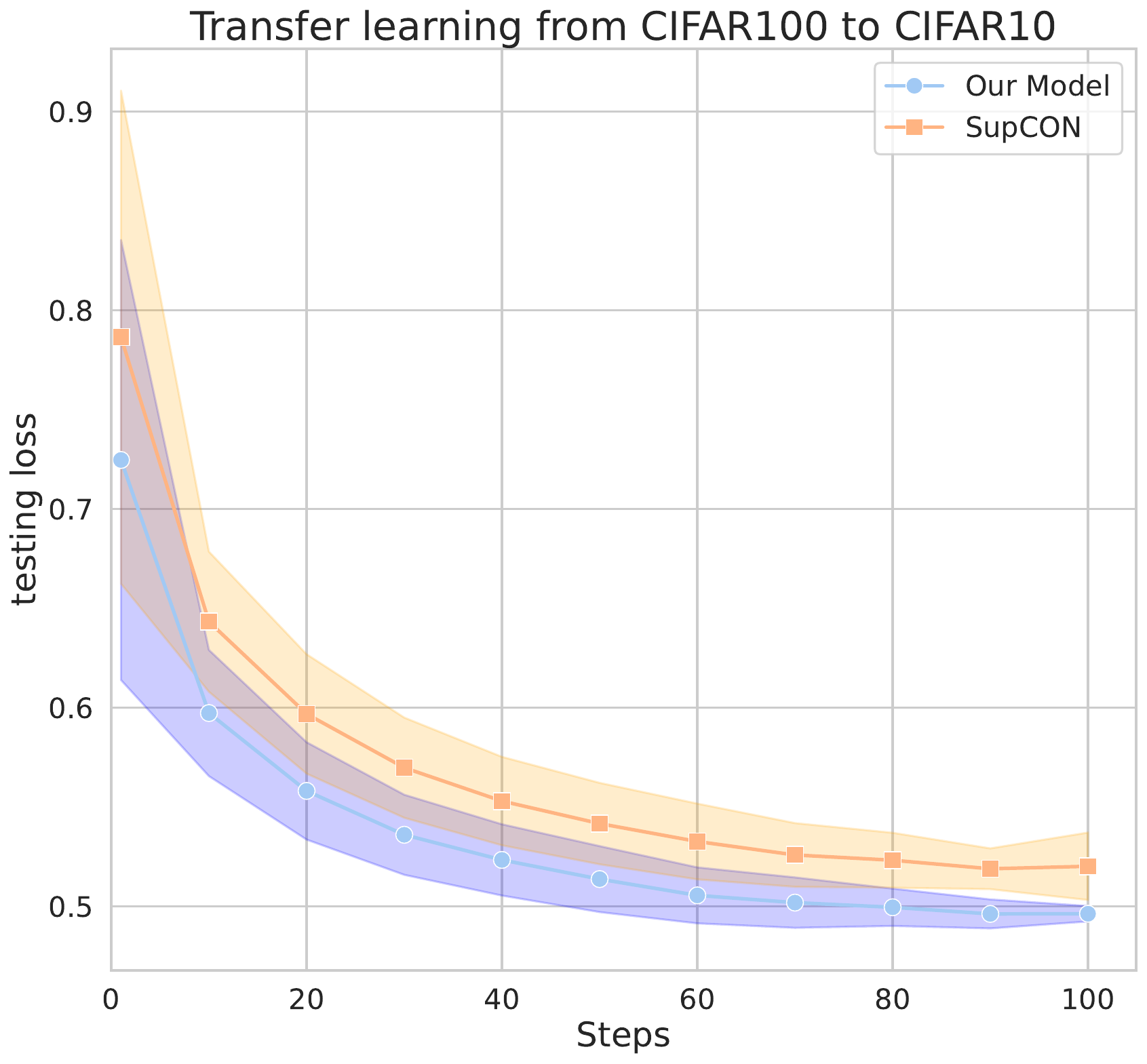}
\caption{Finetuning evaluation accuracy of transfer learning between CIFAR10 and CIFAR100. The blue line represents the ConTeX loss function and the orange line represents the SupCon loss function. We first pretrained the ResNet-50 by both the loss function and both datasets (CIFAR10 and CIFAR100) for 1000 epochs. Then we finetune the network with the other dataset for 100 epochs.}
\label{fig:transferLearning}
\end{figure*}

As we mentioned in the last section, many biases are uncategorizable. Besides the faster convergence speed compared to unsupervised contrastive learning, this section will discuss the potential biases as a side effect of supervised contrastive learning. 
As evidenced in Fig.~\ref{fig:mislead_aug}, post-random augmentation, the last three image groups appear strikingly similar. The present supervised contrastive loss function, which lacks a boundary definition for learning similarities among samples of the same label class, could mislead the network via `shortcuts' engendered by the augmentation process. This becomes particularly plausible in both datasets with smaller image sizes and datasets with large amounts of data in the same category, such as CIFAR10, CIFAR100 and ImageNet. To substantiate this theory, we have designed two experimental setups (Table.~\ref{ntxent pretrained eval} and Fig.~\ref{fig:transferLearning}), demonstrating the discrepancy between our loss function, which circumvents this misleading phenomenon as per Eq.~\ref{loss part 2}, and the SupCon loss function.

In our first experiment, we directly evaluated the CIFAR10, CIFAR100 and ImageNet pretrained networks using unsupervised NT-Xent loss from SimCLR~\cite{SimCLR,NT-Xent}. A high loss value signifies that the pair of images, randomly augmented from the same source image, are predicted as dissimilar (indicating a greater distance in latent space). As illustrated in Table.~\ref{ntxent pretrained eval}, the evaluation loss values for CIFAR10 and CIFAR100, when derived from SupCon, are considerably elevated compared to our approach for both training and testing datasets. This indicates that our method effectively curtails the model's reliance on label shortcuts.

To delve deeper into the tangible repercussions of label shortcuts, we devised a transfer learning experiment using CIFAR10 and CIFAR100. We postulate that training influenced by shortcuts can introduce bias, negatively affecting the model's generalization capability. Fig.\ref{fig:transferLearning} aligns closely with the findings in Table.~\ref{ntxent pretrained eval}, further reinforcing the notion that our ConTeX loss function offers superior generalization capabilities and underscores the detrimental effect of label shortcuts.

\begin{table}[h]
    
  \caption{Transfer learning accuracy comparison based on $32\times32$ size datasets and $224\times224$ datasets. All models are ResNet-50. The first two models are pretrained by ImageNet$32\times32$ dataset~\cite{imagenet32} for 100 epochs and then finetuned on CIFAR10 and CIFAR100 datasets. The second two models are pretrained by Imagenet~\cite{imagenet2012} for 100 epochs and finetuned on Caltech-101~\cite{caltech101} and Caltech-256~\cite{caltech256} datasets. The Caltech datasets' performance are calculated by mean-per-class accuracy.}
  \label{transfer_imagenet32}
  \centering
  \begin{tabular}{cccc}
    \toprule
    Pretrain Dataset & Finetune Dataset &SupCon~\cite{supcon} &Ours\\
    \midrule
    %Linear Evaluating on ImageNet$32\times32$& 39.04 & \textbf{41.85}\\
    ImageNet-32& CIFAR 10& 94.42& \textbf{95.20} \\
    ImageNet-32& CIFAR 100 & 76.32 & \textbf{78.21}\\
    ImageNet & Caltech-101 & 85.55 & \textbf{85.95}\\
    ImageNet & Caltech-256 & 68.49 & \textbf{76.82}\\
    \bottomrule
  \end{tabular}
\end{table}

\subsection{Transfer learning}

To further substantiate the superior generalization capabilities of the ConTeX loss function, particularly in the context of both small image-size datasets and large categorical-size datasets, we carried out another transfer learning experiment. We initiated the pretraining process using the ImageNet$32\times32$ dataset~\cite{imagenet32}, which matches the ImageNet ILSVRC-2012~\cite{imagenet2012} in terms of the total number of data samples and classes. Subsequently, we finetuned the models on CIFAR10 and CIFAR100 datasets after pretrained on ImageNet$32\times32$. The linear evaluation outcomes on the ImageNet$32\times32$ dataset vouch for the efficacy of ConTeX on larger datasets (Table.~\ref{transfer_imagenet32}). Meanwhile, the considerable performance gap during transfer learning across different datasets further highlights the exceptional generalization capacity of our approach.

\subsection{Training details}
All our experiments are trained with three RTX 6000 graphic cards under the PyTorch environment. For performance comparison, we proposed 1000 epochs on pretraining on CIFAR10 and CIFAR100 datasets with a batch size of 2048. We proposed 100 epochs on pretraining the ImageNet dataset with a batch size of 1440 employing mixed precision (combining 16-bit and 32-bit floating-point types) to facilitate larger batch sizes and accelerate computations. All the experiments are trained with the ResNet-50($1\times$) backbone. Similar to SupCon~\cite{supcon}, we use $\tau= 0.1$ as temperature. We use $\lambda = 0.7$ in our proposed loss function for all the experiments. We incorporated a two-layer Multilayer Perceptron (MLP) as the projection head. During the linear evaluation phase, we discarded the projection head and fine-tuned a linear layer, freezing all other parameters. For the experiment in Table.~\ref{ntxent pretrained eval}, we stop updating the parameters and only calculate the NT-Xent loss with different random augmented input images. For the CIFAR10 and CIFAR100 transfer learning experiments, we reused the hyperparameters from the performance comparison experiment, employing a batch size of 256 during the finetuning phase. Further results and hyperparameter details can be found in the Supplementary section. The associated code will be made publicly accessible upon acceptance of this paper.

\subsubsection{Pretraining stage}
In the pretraining phase with the ImageNet dataset, we leveraged a batch size of 1440, a learning rate of 0.5, and training to 100 epochs. When working with smaller image-size datasets, such as CIFAR10 and CIFAR100, we employed a batch size of 2048, a learning rate of 1, and pretrained for 1000 epochs. Across all pretraining experiments, we implemented a cosine learning rate decay and a temperature parameter of 0.1. Our chosen optimizer was SGD, with a momentum parameter set at 0.9.
\subsubsection{Linear evaluate and finetune}
During the linear evaluation phase, we used a batch size of 512 with a learning rate of 0.2 (calculated as $0.1 \times \text{BatchSize} / 256$), over a training period of 100 epochs. Though in typical scenarios, a duration of 30 epochs should suffice. For the finetuning phase, we set a batch size of 256 with a learning rate of 0.2 (calculated as $0.05 \times \text{BatchSize} / 256$) and extended the training to 350 epochs. The chosen optimizer remained SGD, maintaining a momentum parameter of 0.9 for both stages.

\section{Comparative Analysis and Ablation Study}
\subsection{Convergence speed experiments}

\begin{figure}[h]
\centering
\includegraphics[width=.24\textwidth]{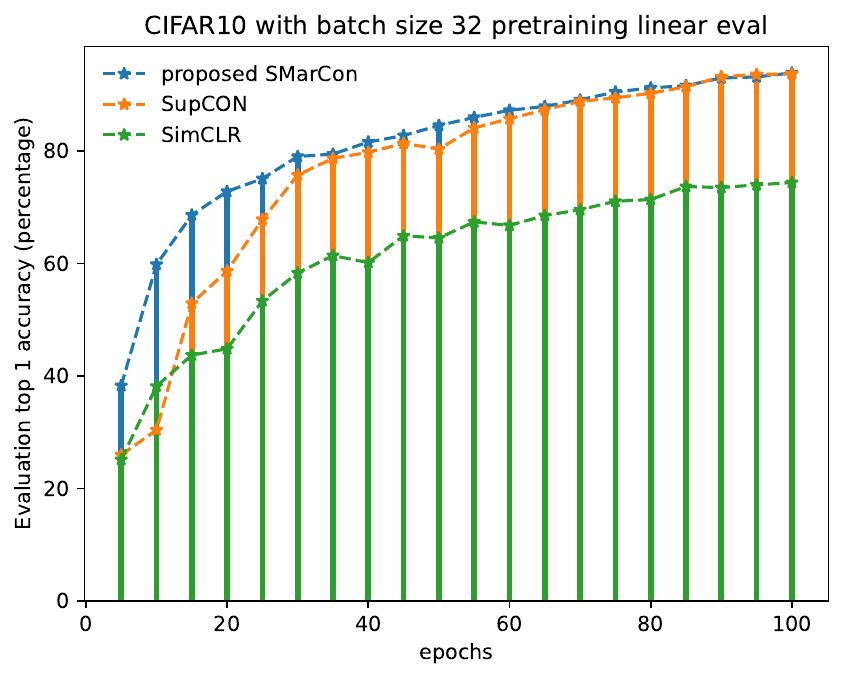}\hfill
\includegraphics[width=.24\textwidth]{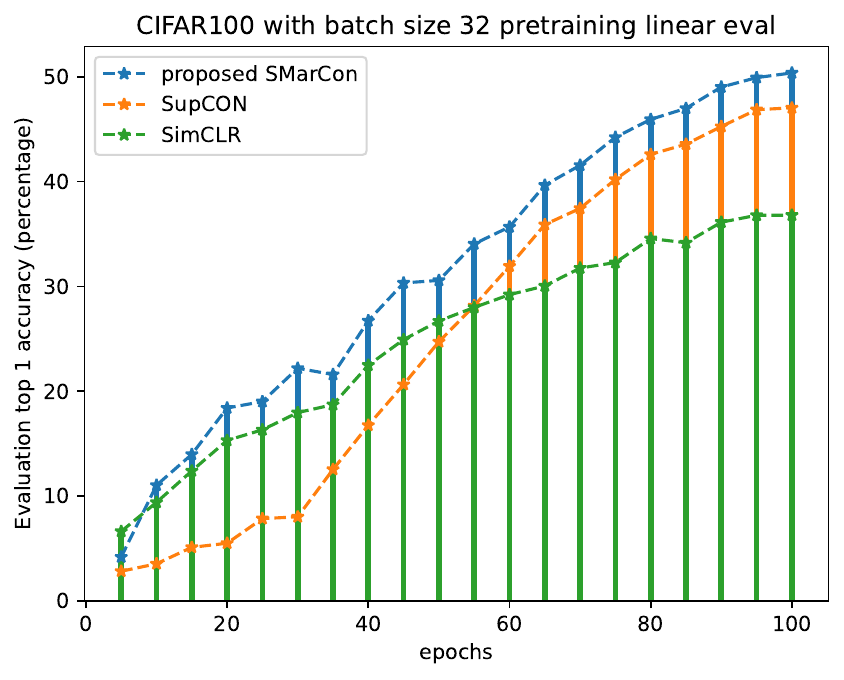}\hfill
\caption{Comparison of linear evaluation accuracy on different training epochs on ConTeX(SMarCon), SupCon~\cite{supcon} and SimCLR~\cite{SimCLR} with a batch size of 32 on CIFAR10 and CIFAR100 datasets. Noting that the lower accuracy on the CIFAR100 dataset is due to the batch size being significantly smaller than the number of classes, a typical challenge encountered in supervised contrastive learning.}
\label{fig:convergence}
\end{figure}

\begin{table*}[!th]
  \caption{Linear evaluation on different weight parameter $\lambda$. All models are pretrained for 100 epochs on the CIFAR10 and CIFAR100 datasets. The other hyperparameters are the same. When $\lambda = 0.0$, the loss function is totally unsupervised, when $lambda = 1.0$, the loss function is fully supervised.}
  \label{weight accuracy}
  \centering
  \begin{tabular}{cccccccccccc}
    \toprule
    Weight($\lambda$)& 0.0 & 0.1 & 0.2 & 0.3 & 0.4 &0.5 & 0.6 & 0.7 & 0.8 & 0.9 & 1.0\\
    \midrule
    CIFAR10&76.6&84.66&90.35&91.84&92.48&92.78&93.68&93.42&92.88&92.26&91.95\\
    CIFAR100&49.25&55.79&61.82&67.67&69.73&72.23&73.53&73.61&73.76&73.12&73.24\\

    \bottomrule
  \end{tabular}
\end{table*}

Converging with a small batch size is always challenging for contrastive learning due to limited negative samples. However, faster convergence can train more robust models and manage larger datasets with limited training resources,  it also accelerates the prototyping process. This rapid cycle of development facilitates the swift refinement of models, making experimentation and innovation efforts far more efficient and effective. Our investigation into various contrastive learning methods with limited small batch sizes demonstrated that our proposed method converges significantly more quickly than SupCon in the early stage of training, as illustrated in Fig.~\ref{fig:convergence}. In the first part of our loss function (Eq.~\ref{loss part 1}), we capitalized on the contrast between positive and negative pairs. The divergence between our proposed ConTeX loss and the conventional SupCon loss progressively widens with an increasing number of training epochs. Our ConTeX loss function not only bolsters the learning process by adeptly differentiating between positive and negative sample pairs, but it also accelerates the speed of convergence markedly. This enhanced convergence efficiency exemplifies the ConTeX loss function's competence in harnessing the intrinsic contrastive potential of the data, thereby accelerating the learning direction of the model. It is worth noting that this accelerated convergence does not compromise the model's performance, affirming the balance between speed and accuracy maintained by the ConTeX loss function. In our CIFAR10 dataset experiment, both our proposed contrastive learning method (ConTeX) and SupCon achieved an accuracy of 93.7\% at epoch 100. In contrast, SimCLR, due to its small batch size with limited negative samples and lack of label information, showed a notable performance gap. The CIFAR100 dataset experiment posed a greater challenge for supervised contrastive frameworks, as fewer positive pairs were available for each of the 100 classes, with a batch size of only 32. This predicament was evident in the performance difference between SimCLR and SupCon; the SupCon framework managed to match the accuracy of the self-supervised learning framework on epoch 50. In contrast, our framework consistently demonstrated superior performance throughout almost the entire training period. Additionally, the left portion of Fig.~\ref{fig:stdmean} depicts our proposed method's mean performance across different batch sizes, which further evidences its propensity for faster convergence regardless of the number of positive pairs available.

\begin{figure}[h]
\centering
\includegraphics[width=.24\textwidth]{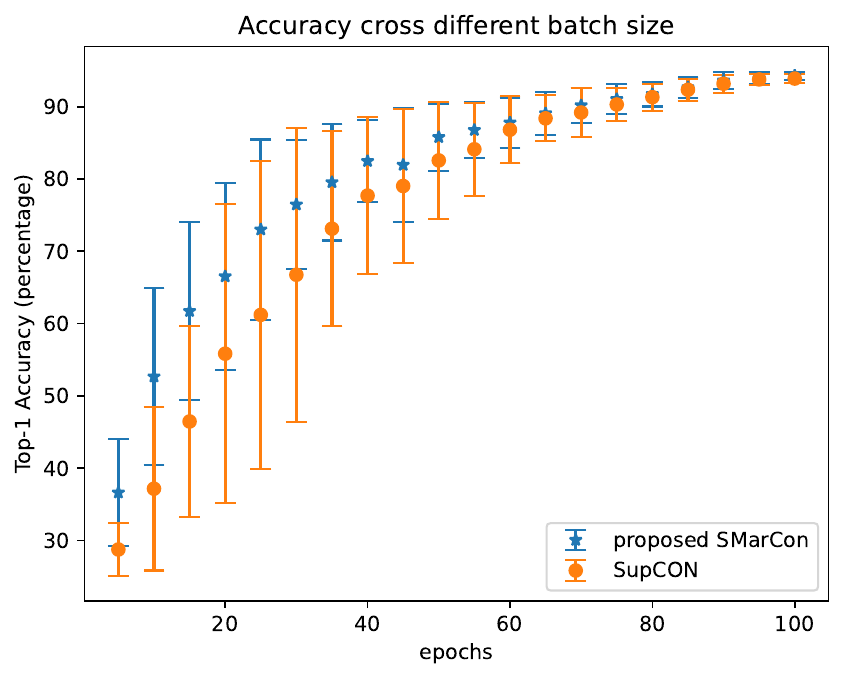}\hfill
\includegraphics[width=.24\textwidth]{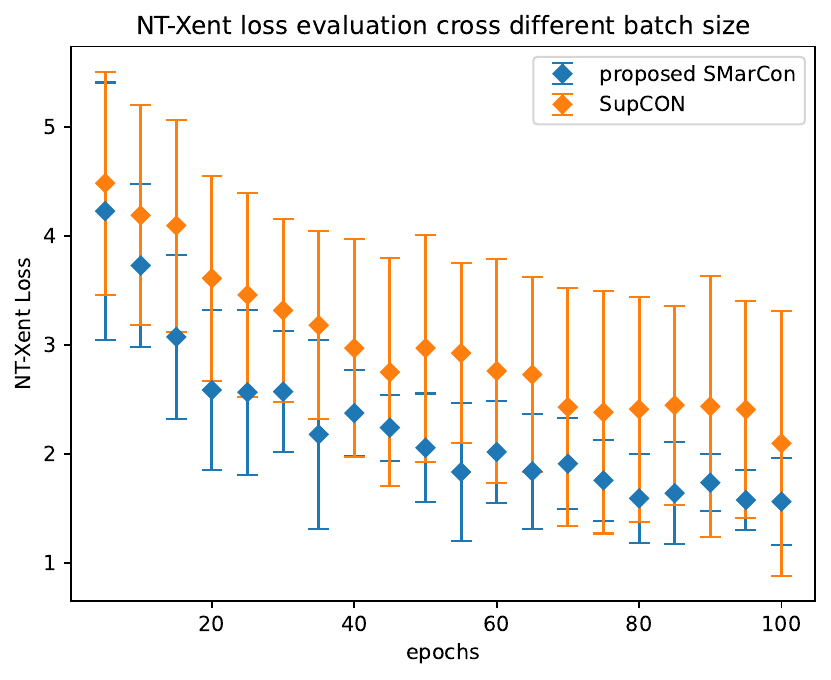}\hfill
\caption{\textbf{Left:} The standard deviation and mean linear evaluation accuracy cross different batch sizes (16, 24, 32, 64, 128, 256) during 100 epochs pertaining. \textbf{Right:} The standard deviation and mean NT-Xent loss~\cite{NT-Xent} cross different batch sizes (16, 24, 32, 64, 128, 256) during 100 epochs pretraining.}
\label{fig:stdmean}
\end{figure}

\subsection{NT-Xent loss evaluation implementation}
In this part, we aimed to use this robust loss function to detail the distance between the anchor and the \textit{self positive} sample ($\Tilde{x_{ni}}$ and $\Tilde{x_{nj}}$). The NT-Xent loss~\cite{SimCLR} is an efficient method to calculate the distance across the whole batch data samples. On the right part of Fig.~\ref{fig:stdmean}, we can observe the NT-Xent loss reveals a downward trend in both methods over time in the right section of Fig.~\ref{fig:stdmean}. This trend indicates that both supervised contrastive methods are progressively extracting knowledge from both the context and inherent features within the data samples. Throughout the entire pretraining phase, ConTeX consistently maintains a lower self-contrast (NT-Xent) loss compared to SupCon, further demonstrating its effectiveness.

\subsection{Analysis of Weight Parameter ($\lambda$)}

This section presents an ablation study examining the impact of varying the weight parameter, $\lambda$, on our combined loss function. In Table.~\ref{weight accuracy}, the model exclusively employs the unsupervised component of our loss function (Eq.~\ref{loss part 2}) at $\lambda = 0.0$. At $\lambda = 1.0$, the model relies solely on the supervised component of our loss function (Eq.~\ref{loss part 1}), maximizing the use of context information. From the results on both datasets, it evidents that a $\lambda$ value between 0.6 and 0.8 yields optimal performance. This suggests a balanced contribution from both supervised and unsupervised components of our loss function is beneficial.

\subsection{Effect of Batch size}
% table performance on each batch size
\begin{table}[h]
  \caption{Linear evaluation on different batch sizes. All models are pretrained for 100 epochs on the CIFAR10 dataset. The hyperparameters are the same across two different loss functions.}
  \label{bs accuracy}
  \centering
  \begin{tabular}{ccccccc}
    \toprule
    Batch size& 16 &24&32&64& 128&256\\
    \midrule
    SupCon~\cite{supcon}&\textbf{93.01}&92.77&93.80&94.27&94.66&94.68\\
    Ours& \textbf{93.02}&\textbf{93.08}&\textbf{94.07}&\textbf{94.61}&\textbf{94.93}&\textbf{95.08}\\

    \bottomrule
  \end{tabular}
\end{table}

We also delved into the impact of varying batch sizes during the pretraining phase. As demonstrated in Table.~\ref{bs accuracy}, accuracy generally trends upward as batch size increases, revealing that both loss functions benefit from larger batch sizes. Despite this, our proposed loss function consistently outperforms SupCon with relatively smaller batch sizes.

Further clarity can be gleaned from Fig.~\ref{fig:stdmean}, where both sections highlight the high robustness and stability of our proposed loss function, indicated by lower standard deviations. These standard deviations are derived from results across different batch sizes within the same training epoch.

In the left section of Fig.~\ref{fig:stdmean}, the standard deviations of our method are closely clustered around the mean, unlike those of SupCon. This suggests that our method is less influenced by batch size variation during the pretraining period.

Looking at the right section of Fig.~\ref{fig:stdmean} we observe significantly lower standard deviations for our loss function. This underlines its consistent learning performance, which is not influenced by batch size. This evidence shows the stability of our loss function in comprehending the internal information from data samples irrespective of batch size.

\section{Conclusion}
We have presented our novel Context-enriched Contrastive (ConTeX) loss function, which harnesses the power of supervised contrastive learning while addressing its limitations. By introducing two constraints and focusing on both context(label) and self-positive pairs, our method has significantly improved the efficiency and generalization capabilities of contrastive learning models. Furthermore, we highlighted and proved the impact of learning from shortcuts resulting from current supervised contrastive learning frameworks, which our method effectively mitigates. This underscores the value of our approach in preventing potential biases and improving the robustness of the model. Our study is an important step towards learning the `fairness' feature from data. As part of our future work, we aim to explore additional ways to further enhance the performance and generalization capabilities of our ConTeX loss function. We hope that our findings will inspire further research and advancements in the field of contrastive learning.

\bibliographystyle{unsrt}
\small
\bibliography{reference}

@article{SimCLR,
  author       = {Ting Chen and
                  Simon Kornblith and
                  Mohammad Norouzi and
                  Geoffrey E. Hinton},
  title        = {A Simple Framework for Contrastive Learning of Visual Representations},
  journal      = {CoRR},
  volume       = {abs/2002.05709},
  year         = {2020},
  url          = {https://arxiv.org/abs/2002.05709},
  eprinttype    = {arXiv},
  eprint       = {2002.05709},
  bibsource    = {dblp computer science bibliography, https://dblp.org}
}

@article{Simsiam,
  author       = {Xinlei Chen and
                  Kaiming He},
  title        = {Exploring Simple Siamese Representation Learning},
  journal      = {CoRR},
  volume       = {abs/2011.10566},
  year         = {2020},
  url          = {https://arxiv.org/abs/2011.10566},
  eprinttype    = {arXiv},
  eprint       = {2011.10566},
  bibsource    = {dblp computer science bibliography, https://dblp.org}
}

@article{MOCO,
  author       = {Kaiming He and
                  Haoqi Fan and
                  Yuxin Wu and
                  Saining Xie and
                  Ross B. Girshick},
  title        = {Momentum Contrast for Unsupervised Visual Representation Learning},
  journal      = {CoRR},
  volume       = {abs/1911.05722},
  year         = {2019},
  url          = {http://arxiv.org/abs/1911.05722},
  eprinttype    = {arXiv},
  eprint       = {1911.05722},
  bibsource    = {dblp computer science bibliography, https://dblp.org}
}

@article{metaPOT,
   author = {Susan Zhang and others},
   doi = {10.48550/arxiv.2205.01068},
   month = {5},
   title = {OPT: Open Pre-trained Transformer Language Models},
   url = {https://arxiv.org/abs/2205.01068v4},
   year = {2022},
}

@article{Nt-Xent_Upper_Bound,
   author = {Wilhelm Ågren wagren},
   doi = {10.1109/CVPR.2018.00393},
   title = {THE NT-XENT LOSS UPPER BOUND AN UPPER BOUND FOR AVERAGE SIMILARITY},
   year = {2022},
}

@article{DeepLearning,
   author = {Yann Lecun and Yoshua Bengio and Geoffrey Hinton},
   doi = {10.1038/nature14539},
   issn = {1476-4687},
   issue = {7553},
   journal = {Nature 2015 521:7553},
   month = {5},
   pages = {436-444},
   pmid = {26017442},
   publisher = {Nature Publishing Group},
   title = {Deep learning},
   volume = {521},
   url = {https://www.nature.com/articles/nature14539},
   year = {2015},
}

@article{CNN,
   author = {Yann LeCun and Patrick Haffner and Léon Bottou and Yoshua Bengio},
   doi = {10.1007/3-540-46805-6_19/COVER},
   isbn = {3540667229},
   issn = {16113349},
   journal = {Lecture Notes in Computer Science (including subseries Lecture Notes in Artificial Intelligence and Lecture Notes in Bioinformatics)},
   pages = {319-345},
   publisher = {Springer Verlag},
   title = {Object recognition with gradient-based learning},
   volume = {1681},
   url = {https://link.springer.com/chapter/10.1007/3-540-46805-6_19},
   year = {1999},
}

@article{Transformers,
   author = {Ashish Vaswani and Noam Shazeer and Niki Parmar and Jakob Uszkoreit and Llion Jones and Aidan N. Gomez and Łukasz Kaiser and Illia Polosukhin},
   doi = {10.48550/arxiv.1706.03762},
   issn = {10495258},
   journal = {Advances in Neural Information Processing Systems},
   month = {6},
   pages = {5999-6009},
   publisher = {Neural information processing systems foundation},
   title = {Attention Is All You Need},
   volume = {2017-December},
   url = {https://arxiv.org/abs/1706.03762v5},
   year = {2017},
}

@inproceedings{supcon,
	title = {Supervised {Contrastive} {Learning}},
	volume = {33},
	url = {https://proceedings.neurips.cc/paper_files/paper/2020/file/d89a66c7c80a29b1bdbab0f2a1a94af8-Paper.pdf},
	booktitle = {Advances in {Neural} {Information} {Processing} {Systems}},
	publisher = {Curran Associates, Inc.},
	author = {Khosla, Prannay and others},
	editor = {Larochelle, H. and Ranzato, M. and Hadsell, R. and Balcan, M. F. and Lin, H.},
	year = {2020},
	pages = {18661--18673},
}

@misc{resnet,
      title={Deep Residual Learning for Image Recognition}, 
      author={Kaiming He and Xiangyu Zhang and Shaoqing Ren and Jian Sun},
      year={2015},
      eprint={1512.03385},
      archivePrefix={arXiv},
      primaryClass={cs.CV}
}

@inproceedings{DM_large_margin,
	title = {Distance {Metric} {Learning} for {Large} {Margin} {Nearest} {Neighbor} {Classification}},
	volume = {18},
	url = {https://proceedings.neurips.cc/paper/2005/hash/a7f592cef8b130a6967a90617db5681b-Abstract.html},
	urldate = {2023-05-11},
	booktitle = {Advances in {Neural} {Information} {Processing} {Systems}},
	publisher = {MIT Press},
	author = {Weinberger, Kilian Q and Blitzer, John and Saul, Lawrence},
	year = {2005},
}

@inproceedings{NT-Xent,
	title = {Improved {Deep} {Metric} {Learning} with {Multi}-class {N}-pair {Loss} {Objective}},
	volume = {29},
	url = {https://papers.nips.cc/paper_files/paper/2016/hash/6b180037abbebea991d8b1232f8a8ca9-Abstract.html},
	urldate = {2023-05-11},
	booktitle = {Advances in {Neural} {Information} {Processing} {Systems}},
	publisher = {Curran Associates, Inc.},
	author = {Sohn, Kihyuk},
	year = {2016},
}

@inproceedings{DM_Relative_Comparisons,
	title = {Learning a {Distance} {Metric} from {Relative} {Comparisons}},
	volume = {16},
	url = {https://proceedings.neurips.cc/paper_files/paper/2003/hash/d3b1fb02964aa64e257f9f26a31f72cf-Abstract.html},
	urldate = {2023-05-14},
	booktitle = {Advances in {Neural} {Information} {Processing} {Systems}},
	publisher = {MIT Press},
	author = {Schultz, Matthew and Joachims, Thorsten},
	year = {2003},
}

@inproceedings{rank_sim,
	address = {Berlin, Heidelberg},
	series = {Lecture {Notes} in {Computer} {Science}},
	title = {Large {Scale} {Online} {Learning} of {Image} {Similarity} through {Ranking}},
	isbn = {978-3-642-02172-5},
	doi = {10.1007/978-3-642-02172-5_2},
	language = {en},
	booktitle = {Pattern {Recognition} and {Image} {Analysis}},
	publisher = {Springer},
	author = {Chechik, Gal and Sharma, Varun and Shalit, Uri and Bengio, Samy},
	editor = {Araujo, Helder and Mendonça, Ana Maria and Pinho, Armando J. and Torres, María Inés},
	year = {2009},
	pages = {11--14},
}

@inproceedings{triplet_loss,
	address = {Boston, MA, USA},
	title = {{FaceNet}: {A} unified embedding for face recognition and clustering},
	isbn = {978-1-4673-6964-0},
	shorttitle = {{FaceNet}},
	url = {http://ieeexplore.ieee.org/document/7298682/},
	doi = {10.1109/CVPR.2015.7298682},
	language = {en},
	urldate = {2023-05-15},
	booktitle = {2015 {IEEE} {Conference} on {Computer} {Vision} and {Pattern} {Recognition} ({CVPR})},
	publisher = {IEEE},
	author = {Schroff, Florian and Kalenichenko, Dmitry and Philbin, James},
	month = jun,
	year = {2015},
	pages = {815--823},
}

@INPROCEEDINGS{imagenet,
  author={Deng, Jia and Dong, Wei and Socher, Richard and Li, Li-Jia and Kai Li and Li Fei-Fei},
  booktitle={2009 IEEE Conference on Computer Vision and Pattern Recognition}, 
  title={ImageNet: A large-scale hierarchical image database}, 
  year={2009},
  volume={},
  number={},
  pages={248-255},
  doi={10.1109/CVPR.2009.5206848}}

@article{cifar,
author = {Krizhevsky, Alex},
year = {2012},
month = {05},
pages = {},
title = {Learning Multiple Layers of Features from Tiny Images},
journal = {University of Toronto}
}

@misc{caltech101, title={Caltech 101}, DOI={10.22002/D1.20086}, abstractNote={Pictures of objects belonging to 101 categories. About 40 to 800 images per category. Most categories have about 50 images. Collected in September 2003 by Fei-Fei Li, Marco Andreetto, and Marc'Aurelio Ranzato. The size of each image is roughly 300 x 200 pixels. We have carefully clicked outlines of each object in these pictures, these are included under the 'Annotations.tar'. There is also a MATLAB script to view the annotations, 'show_annotations.m'.}, publisher={CaltechDATA}, author={Li, Fei-Fei and Andreeto, Marco and Ranzato, Marc'Aurelio and Perona, Pietro}, year={2022}, month={Apr} }

@misc{caltech256, title={Caltech 256}, DOI={10.22002/D1.20087}, abstractNote={We introduce a challenging set of 256 object categories containing a total of 30607 images. The original Caltech-101 was collected by choosing a set of object categories, downloading examples from Google Images and then manually screening out all images that did not fit the category. Caltech-256 is collected in a similar manner with several improvements: a) the number of categories is more than doubled, b) the minimum number of images in any category is increased from 31 to 80, c) artifacts due to image rotation are avoided and d) a new and larger clutter category is introduced for testing background rejection. We suggest several testing paradigms to measure classification performance, then benchmark the dataset using two simple metrics as well as a state-of-the-art spatial pyramid matching algorithm. Finally we use the clutter category to train an interest detector which rejects uninformative background regions.}, publisher={CaltechDATA}, author={Griffin, Gregory and Holub, Alex and Perona, Pietro}, year={2022}, month={Apr} }

@inproceedings{LargeMargin,
  title={Large-Margin Softmax Loss for Convolutional Neural Networks},
  author={Weiyang Liu and Yandong Wen and Zhiding Yu and Meng Yang},
  booktitle={International Conference on Machine Learning},
  year={2016}
}

@article{imagenet2012,
Author = {Olga Russakovsky and Jia Deng and Hao Su and Jonathan Krause and Sanjeev Satheesh and Sean Ma and Zhiheng Huang and Andrej Karpathy and Aditya Khosla and Michael Bernstein and Alexander C. Berg and Li Fei-Fei},
Title = {{ImageNet Large Scale Visual Recognition Challenge}},
Year = {2015},
journal   = {International Journal of Computer Vision (IJCV)},
doi = {10.1007/s11263-015-0816-y},
volume={115},
number={3},
pages={211-252}
}

@misc{imagenet32,
      title={A Downsampled Variant of ImageNet as an Alternative to the CIFAR datasets}, 
      author={Patryk Chrabaszcz and Ilya Loshchilov and Frank Hutter},
      year={2017},
      eprint={1707.08819},
      archivePrefix={arXiv},
      primaryClass={cs.CV}
}

@inproceedings{hong_unbiased_2021,
	title = {Unbiased {Classification} through {Bias}-{Contrastive} and {Bias}-{Balanced} {Learning}},
	volume = {34},
	url = {https://proceedings.neurips.cc/paper_files/paper/2021/hash/de8aa43e5d5fa8536cf23e54244476fa-Abstract.html},
	urldate = {2023-08-25},
	booktitle = {Advances in {Neural} {Information} {Processing} {Systems}},
	publisher = {Curran Associates, Inc.},
	author = {Hong, Youngkyu and Yang, Eunho},
	year = {2021},
	pages = {26449--26461},
}

@inproceedings{chuang_debiased_2020,
	title = {Debiased {Contrastive} {Learning}},
	volume = {33},
	url = {https://proceedings.neurips.cc/paper/2020/hash/63c3ddcc7b23daa1e42dc41f9a44a873-Abstract.html},
	urldate = {2023-08-25},
	booktitle = {Advances in {Neural} {Information} {Processing} {Systems}},
	publisher = {Curran Associates, Inc.},
	author = {Chuang, Ching-Yao and Robinson, Joshua and Lin, Yen-Chen and Torralba, Antonio and Jegelka, Stefanie},
	year = {2020},
	pages = {8765--8775},
}

@inproceedings{barbano_unbiased_2022,
	title = {Unbiased {Supervised} {Contrastive} {Learning}},
	url = {https://openreview.net/forum?id=Ph5cJSfD2XN},
	language = {en},
	urldate = {2023-08-25},
	author = {Barbano, Carlo Alberto and Dufumier, Benoit and Tartaglione, Enzo and Grangetto, Marco and Gori, Pietro},
	month = sep,
	year = {2022},
}

@misc{bahng_learning_2020,
	title = {Learning {De}-biased {Representations} with {Biased} {Representations}},
	url = {http://arxiv.org/abs/1910.02806},
	urldate = {2023-08-25},
	publisher = {arXiv},
	author = {Bahng, Hyojin and Chun, Sanghyuk and Yun, Sangdoo and Choo, Jaegul and Oh, Seong Joon},
	month = jun,
	year = {2020},
	note = {arXiv:1910.02806 [cs, stat]},
}

@inproceedings{bmnist,
	title = {Learning {De}-biased {Representations} with {Biased} {Representations}},
	url = {https://proceedings.mlr.press/v119/bahng20a.html},
	language = {en},
	urldate = {2023-09-21},
	booktitle = {Proceedings of the 37th {International} {Conference} on {Machine} {Learning}},
	publisher = {PMLR},
	author = {Bahng, Hyojin and Chun, Sanghyuk and Yun, Sangdoo and Choo, Jaegul and Oh, Seong Joon},
	month = nov,
	year = {2020},
	note = {ISSN: 2640-3498},
	pages = {528--539},
}

@article{cleleba,
	title = {Deep {Learning} {Face} {Attributes} in the {Wild}},
	language = {en},
	author = {Liu, Ziwei and Luo, Ping and Wang, Xiaogang and Tang, Xiaoou},
}

@inproceedings{utkface,
	address = {Honolulu, HI},
	title = {Age {Progression}/{Regression} by {Conditional} {Adversarial} {Autoencoder}},
	isbn = {978-1-5386-0457-1},
	url = {http://ieeexplore.ieee.org/document/8099946/},
	doi = {10.1109/CVPR.2017.463},
	language = {en},
	urldate = {2023-09-21},
	booktitle = {2017 {IEEE} {Conference} on {Computer} {Vision} and {Pattern} {Recognition} ({CVPR})},
	publisher = {IEEE},
	author = {Zhang, Zhifei and Song, Yang and Qi, Hairong},
	month = jul,
	year = {2017},
	pages = {4352--4360},
}

@inproceedings{lnl,
	address = {Long Beach, CA, USA},
	title = {Learning {Not} to {Learn}: {Training} {Deep} {Neural} {Networks} {With} {Biased} {Data}},
	isbn = {978-1-72813-293-8},
	shorttitle = {Learning {Not} to {Learn}},
	url = {https://ieeexplore.ieee.org/document/8953715/},
	doi = {10.1109/CVPR.2019.00922},
	language = {en},
	urldate = {2023-09-21},
	booktitle = {2019 {IEEE}/{CVF} {Conference} on {Computer} {Vision} and {Pattern} {Recognition} ({CVPR})},
	publisher = {IEEE},
	author = {Kim, Byungju and Kim, Hyunwoo and Kim, Kyungsu and Kim, Sungjin and Kim, Junmo},
	month = jun,
	year = {2019},
	pages = {9004--9012},
}

@inproceedings{end,
	address = {Nashville, TN, USA},
	title = {{EnD}: {Entangling} and {Disentangling} deep representations for bias correction},
	isbn = {978-1-66544-509-2},
	shorttitle = {{EnD}},
	url = {https://ieeexplore.ieee.org/document/9577751/},
	doi = {10.1109/CVPR46437.2021.01330},
	language = {en},
	urldate = {2023-09-21},
	booktitle = {2021 {IEEE}/{CVF} {Conference} on {Computer} {Vision} and {Pattern} {Recognition} ({CVPR})},
	publisher = {IEEE},
	author = {Tartaglione, Enzo and Barbano, Carlo Alberto and Grangetto, Marco},
	month = jun,
	year = {2021},
	pages = {13503--13512},
}

@inproceedings{di,
	address = {Seattle, WA, USA},
	title = {Towards {Fairness} in {Visual} {Recognition}: {Effective} {Strategies} for {Bias} {Mitigation}},
	isbn = {978-1-72817-168-5},
	shorttitle = {Towards {Fairness} in {Visual} {Recognition}},
	url = {https://ieeexplore.ieee.org/document/9156668/},
	doi = {10.1109/CVPR42600.2020.00894},
	language = {en},
	urldate = {2023-09-21},
	booktitle = {2020 {IEEE}/{CVF} {Conference} on {Computer} {Vision} and {Pattern} {Recognition} ({CVPR})},
	publisher = {IEEE},
	author = {Wang, Zeyu and Qinami, Klint and Karakozis, Ioannis Christos and Genova, Kyle and Nair, Prem and Hata, Kenji and Russakovsky, Olga},
	month = jun,
	year = {2020},
	pages = {8916--8925},
}

@misc{LM,
	title = {Don't {Take} the {Easy} {Way} {Out}: {Ensemble} {Based} {Methods} for {Avoiding} {Known} {Dataset} {Biases}},
	shorttitle = {Don't {Take} the {Easy} {Way} {Out}},
	url = {http://arxiv.org/abs/1909.03683},
	doi = {10.48550/arXiv.1909.03683},
	urldate = {2023-09-21},
	publisher = {arXiv},
	author = {Clark, Christopher and Yatskar, Mark and Zettlemoyer, Luke},
	month = sep,
	year = {2019},
	note = {arXiv:1909.03683 [cs]},
}

@inproceedings{rubi,
	title = {{RUBi}: {Reducing} {Unimodal} {Biases} for {Visual} {Question} {Answering}},
	volume = {32},
	shorttitle = {{RUBi}},
	url = {https://proceedings.neurips.cc/paper_files/paper/2019/hash/51d92be1c60d1db1d2e5e7a07da55b26-Abstract.html},
	urldate = {2023-09-21},
	booktitle = {Advances in {Neural} {Information} {Processing} {Systems}},
	publisher = {Curran Associates, Inc.},
	author = {Cadene, Remi and Dancette, Corentin and Ben younes, Hedi and Cord, Matthieu and Parikh, Devi},
	year = {2019},
}

@inproceedings{lff,
	title = {Learning from {Failure}: {De}-biasing {Classifier} from {Biased} {Classifier}},
	volume = {33},
	shorttitle = {Learning from {Failure}},
	url = {https://proceedings.neurips.cc/paper/2020/hash/eddc3427c5d77843c2253f1e799fe933-Abstract.html},
	urldate = {2023-09-21},
	booktitle = {Advances in {Neural} {Information} {Processing} {Systems}},
	publisher = {Curran Associates, Inc.},
	author = {Nam, Junhyun and Cha, Hyuntak and Ahn, Sungsoo and Lee, Jaeho and Shin, Jinwoo},
	year = {2020},
	pages = {20673--20684},
}

@ARTICLE{zeroshotdual,
  author={Yang, Yanhua and Pan, Rui and Li, Xiangyu and Yang, Xu and Deng, Cheng},
  journal={IEEE Transactions on Multimedia}, 
  title={Dual-Stream Contrastive Learning for Compositional Zero-Shot Recognition}, 
  year={2024},
  volume={26},
  number={},
  pages={1909-1919},
  doi={10.1109/TMM.2023.3243674}}

@article{clip,
  author       = {Alec Radford and
                  others},
  title        = {Learning Transferable Visual Models From Natural Language Supervision},
  journal      = {CoRR},
  volume       = {abs/2103.00020},
  year         = {2021},
  url          = {https://arxiv.org/abs/2103.00020},
  eprinttype    = {arXiv},
  eprint       = {2103.00020},
  bibsource    = {dblp computer science bibliography, https://dblp.org}
}

@ARTICLE{pcl,
  author={Tang, Hao and Zhao, Guoshuai and Gao, Jing and Qian, Xueming},
  journal={IEEE Transactions on Multimedia}, 
  title={Personalized Representation With Contrastive Loss for Recommendation Systems}, 
  year={2024},
  volume={26},
  number={},
  pages={2419-2429},
  doi={10.1109/TMM.2023.3295740}}

@ARTICLE{mscl,
  author={Tang, Hao and Zhao, Guoshuai and Wu, Yuxia and Qian, Xueming},
  journal={IEEE Transactions on Multimedia}, 
  title={Multisample-Based Contrastive Loss for Top-K Recommendation}, 
  year={2023},
  volume={25},
  number={},
  pages={339-351},
  doi={10.1109/TMM.2021.3126146}}

@ARTICLE{pinv,
  author={Yang, Yimin and Wu, Q. M. Jonathan and Feng, Xiexing and Akilan, Thangarajah},
  journal={IEEE Transactions on Pattern Analysis and Machine Intelligence}, 
  title={Recomputation of the Dense Layers for Performance Improvement of DCNN}, 
  year={2020},
  volume={42},
  number={11},
  pages={2912-2925},
  doi={10.1109/TPAMI.2019.2917685}}

@ARTICLE{eeg_dcnet,
  author={Deng, Haojin and Wang, Shiqi and Yang, Yimin and Zhao, W.G.Will and Zhang, Hui and Wei, Ruizhong and Wu, Q.M.Jonathan and Lu, Bao-Liang},
  journal={IEEE Transactions on Cognitive and Developmental Systems}, 
  title={Minimizing {EEG} Human Interference: A Study of an Adaptive {EEG} Spatial Feature Extraction with Deep Convolutional Neural Networks}, 
  year={2024},
  volume={},
  number={},
  pages={1-14},
  doi={10.1109/TCDS.2024.3391131}}

\begin{IEEEbiography}[{\includegraphics[width=0.95in,height=1.25in, clip, keepaspectratio]{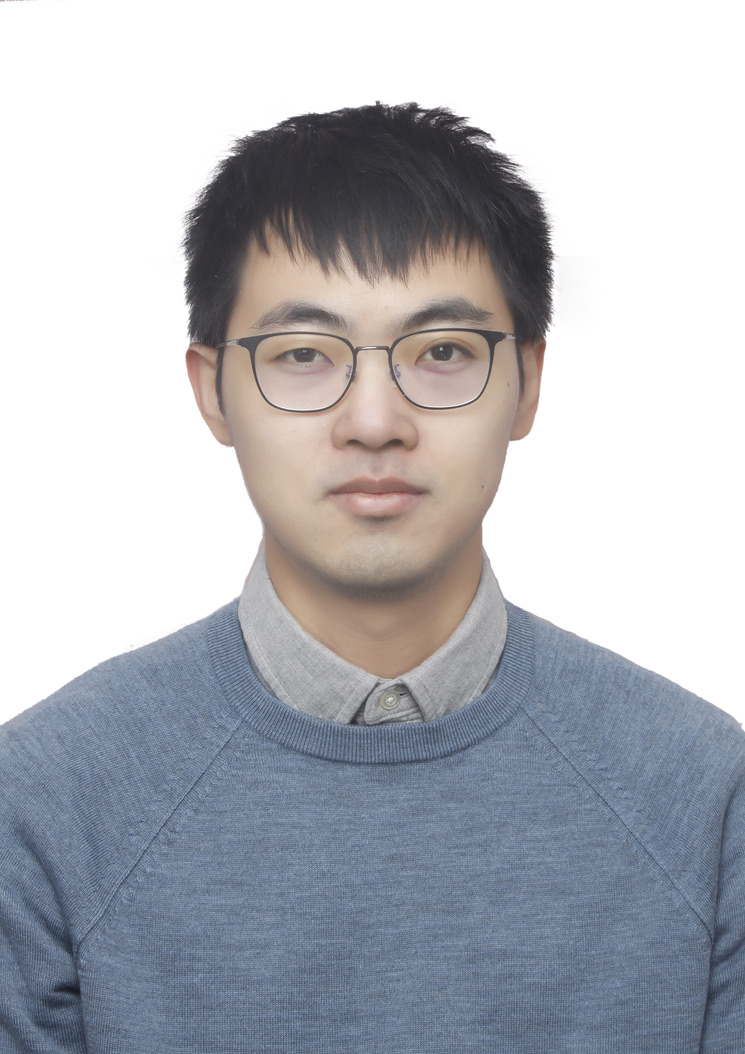}}]{Haojin Deng} received his bachelor's degree (with honours) in Computer Science from Carleton University, Canada in 2020 and M.Sc degree in Computer Science with Specialization in Artificial Intelligence from Lakehead University in 2022. 

He is currently a Ph.D. candidate in Electrical and Computer Engineering, Western University, Canada. He is also a Graduate Student Researcher in Vector Institute, Canada. His current research interests include brain-computer interface, self-supervised learning and representation learning. He is a reviewer for the IEEE Transactions on Cybernetics, and IEEE Transactions on Circuits and Systems for Video Technology.
\end{IEEEbiography}

\begin{IEEEbiography}[{\includegraphics[width=0.95in,height=1.25in]{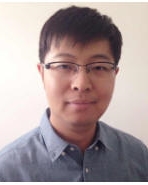}}]{Yimin Yang}
(S’10-M’13-SM’19) received the Ph.D. degree in Pattern Recognition and Intelligent System from the College of Electrical and Information Engineering, Hunan University, Changsha, China, in 2013. 

He is currently an Assistant Professor with the Department of Electrical and Computer Engineering, Western University, Canada. He is also a Faculty Affiliate Member in Vector Institute, Canada. From 2014 to 2018, he was a Post-Doctoral Fellow with the Department of Electrical and Computer Engineering at the University of Windsor, Canada. His current research interests are artificial neural networks, image processing, and data analysis. 

Dr. Yang is an Associate Editor for the IEEE Transactions on Circuits and Systems for Video Technology, Cognitive Computation, and Neurocomputing.  He has been serving as a Reivewer for many international journals of his research field and a Program Committee Member of some international Conferences.  
\end{IEEEbiography}

\end{document}